\documentclass{article} 
\usepackage{iclr2027_conference,times}

\usepackage{amsmath,amsfonts,bm}

\def\eqref#1{equation~\ref{#1}}

\def\1{\bm{1}}

\DeclareMathAlphabet{\mathsfit}{\encodingdefault}{\sfdefault}{m}{sl}
\SetMathAlphabet{\mathsfit}{bold}{\encodingdefault}{\sfdefault}{bx}{n}

\DeclareMathOperator*{\argmax}{arg\,max}
\DeclareMathOperator*{\argmin}{arg\,min}

\usepackage{hyperref}
\usepackage{url}
\usepackage{algorithm}
\usepackage{algpseudocode}
\usepackage{graphicx}
\usepackage{booktabs}
\usepackage{enumitem}
\usepackage{subcaption}
\usepackage{amsmath,amssymb,amsthm}
\usepackage{bbm}
\usepackage{titletoc}
\newtheorem{theorem}{Theorem}[section]
\newtheorem{proposition}[theorem]{Proposition}
\newtheorem{lemma}[theorem]{Lemma}

\newtheorem{assumption}[theorem]{Assumption}
\theoremstyle{definition}
\newtheorem{definition}[theorem]{Definition}
\newtheorem{remark}[theorem]{Remark}

\newcommand{\Pp}{\mathbb{P}}
\newcommand{\bbE}{\mathbb{E}}
\newcommand{\KL}{\operatorname{KL}}
\newcommand{\kl}{\operatorname{kl}}
\newcommand{\Alt}{\operatorname{Alt}}
\newcommand{\pa}{\operatorname{pa}}
\newcommand{\diag}{\operatorname{diag}}
\newcommand{\tr}{\operatorname{tr}}

\newcommand{\bA}{\bm A}
\newcommand{\bSigma}{\bm\Sigma}

\newcommand{\bw}{\bm w}
\newcommand{\bG}{\bm G}

\newcommand{\bh}{\bm h}
\newcommand{\bH}{\bm H}
\newcommand{\bV}{\bm V}
\newcommand{\bx}{\bm x}
\newcommand{\bb}{\bm b}
\newcommand{\bg}{\bm g}
\newcommand{\be}{\bm e}
\newcommand{\cA}{\mathcal A}

\newcommand{\cN}{\mathcal N}

\newcommand{\bX}{\bm X}

\newcommand{\bM}{\bm M}
\newcommand{\bW}{\bm W}
\newcommand{\bI}{\bm I}
\newcommand{\cF}{\mathcal F}

\title{FOCUS:\\ Fixed-Confidence Online Causal Learning
Using Sequential Adaptive Interventions}

\author{Haijie Xu\\
Department of Industrial Engineering\\
Tsinghua University\\
Beijing, China\\
\texttt{xu-hj22@mails.tsinghua.edu.cn}
\And
Chen Zhang\thanks{Corresponding author.}\\
Department of Industrial Engineering\\
Tsinghua University\\
Beijing, China\\
\texttt{zhangchen01@tsinghua.edu.cn}
}

\iclrfinalcopy 
\begin{document}

\maketitle
\lhead{Preprint}

\begin{abstract}
We study fully online fixed-confidence causal discovery without any historical
observational data. Starting from zero samples, the learner sequentially selects
interventions to recover both the causal DAG and its edge weights under a
linear-Gaussian structural equation model. We establish an instance-dependent
lower bound for any $(\epsilon,\delta)$-correct algorithm and propose
\textsc{FOCUS}, which adaptively allocates interventions through an online
max--min game. A key contribution is a computable concentration inequality for
the accumulated KL divergence involving causal parameters shared across
interventions. We prove that \textsc{FOCUS} is $(\epsilon,\delta)$-correct and
that its expected stopping time matches the lower bound in its
$\Theta(\log(1/\delta))$ dependence up to an instance-dependent constant.
Experiments demonstrate improved structure and edge-weight recovery and confirm
the predicted stopping-time trend. Our codes are available on \url{https://anonymous.4open.science/r/FOCUS_code-76E5}
\end{abstract}

\section{Introduction}
\label{sec:introduction}

Causal discovery seeks to uncover the directed relationships that govern a
system and is fundamental to scientific reasoning and decision making
\citep{pearl2009causality,spirtes2000causation,peters2017elements}.
Observational data, however, generally identify a causal directed acyclic graph
(DAG) only up to its Markov equivalence class, leaving many edge directions
unresolved \citep{spirtes2000causation,peters2017elements}. Interventions can
resolve this ambiguity by actively perturbing the system, but are often
substantially more costly than passive observations
\citep{hauser2014two,shanmugam2015learning}. In clinical trials, for example, premature stopping for apparent benefit can
yield exaggerated treatment-effect estimates \citep{montori2005randomized},
whereas unnecessary continuation consumes resources and may expose additional
patients to inferior treatments \citep{pallmann2018adaptive}. A sample-efficient learner therefore needs both an
adaptive intervention strategy and a statistically valid stopping rule:
\textbf{which intervention should be performed next, and when is the evidence
sufficient to return a reliable causal model?}

Existing intervention-design methods do not fully address both questions.
Classical methods typically study graph identification under population-level
observational and interventional information, effectively treating distributions
or edge orientations as known without sampling error
\citep{shanmugam2015learning,kocaoglu2017cost,
squires2020active,choo2023adaptivity}. Recent finite-data methods adaptively
select interventions, but generally operate under a prescribed budget rather
than a target confidence level
\citep{scherrer2021learning,tigas2022interventions,olko2023trust}. The closest
fixed-confidence approach is \citet{elahi2024adaptive}, but it assumes an
available faithful observational distribution, restricts variables to finite
discrete state spaces, and targets only DAG recovery. This leaves open fully
online fixed-confidence learning for continuous systems, where no historical
observational data are available and both causal structure and effect
magnitudes must be learned. The distinction matters because the graph may be
recovered well before its edge weights are estimated accurately.

We address this problem for linear-Gaussian structural equation models (SEM). The
learner starts from zero samples and, at every round, chooses either passive
observation or a single-node intervention, observes one new sample,
and updates its estimate of the DAG, edge weights, and noise variances. Given
an accuracy level $\epsilon$ and confidence level $\delta$, the objective is
to stop as early as possible while returning the true DAG and an edge-weight
matrix whose maximum entrywise error is smaller than $\epsilon$, with
probability at least $1-\delta$. To this end, we introduce \textbf{F}ixed-confidence \textbf{O}nline \textbf{C}ausal learning method
\textbf{U}sing \textbf{S}equential adaptive interventions (\textsc{FOCUS}), which is an online algorithm that does not need any historical data. We first characterize the
instance-dependent information complexity of the problem and show that the
continuous intervention intervals can be reduced without loss to their
endpoints. \textsc{FOCUS} then combines joint maximum-likelihood estimation
with an online max--min game to learn an informative intervention allocation,
which is converted into actual interventions through forced exploration and
cumulative tracking. A data-dependent stopping rule determines when the
accumulated evidence is sufficient. We prove that \textsc{FOCUS} is
$(\epsilon,\delta)$-correct (see Definition \ref{def:correctness}) and that its expected stopping time has the
optimal $\Theta(\log(1/\delta))$ dependence up to an instance-dependent
constant.

\paragraph{Contributions.}
\textbf{1) Fully online formulation and lower bound.}
We formulate fixed-confidence joint recovery of causal structure and edge
weights in linear-Gaussian SEMs without historical observational data, and
derive an instance-dependent information-theoretic lower bound for every
$(\epsilon,\delta)$-correct algorithm.
\textbf{2) Adaptive intervention design.}
We prove that restricting each continuous intervention interval to its two
endpoints preserves the information complexity, and propose \textsc{FOCUS},
which learns and tracks an informative intervention allocation through an
online max--min game.
\textbf{3) Fixed-confidence guarantees.}
We establish that \textsc{FOCUS} is $(\epsilon,\delta)$-correct and derive an
instance-dependent upper bound on its expected stopping time, showing that it
is order-optimal in the confidence level. A technical ingredient is a
computable concentration inequality for the accumulated KL divergence under
adaptive interventions.
\textbf{4) Implementation and experiments.}
We provide finite implementations of the optimization procedures used by
\textsc{FOCUS}. Experiments across graphs of different sizes and densities
demonstrate improved structure and edge-weight recovery and confirm the
predicted dependence of the stopping time on $\log(1/\delta)$.

\section{Related Work}
\label{sec:related_work}

\paragraph{Non-adaptive causal discovery.}
Causal structures can be learned from observational data using constraint-based,
score-based, or continuous optimization methods
\citep{spirtes2000causation,chickering2002optimal,zheng2018dags}.
Under the standard Markov and faithfulness assumptions, however, observational
data generally identify only a Markov equivalence class. In particular, a
linear-Gaussian SEM with unrestricted noise variances is not identifiable from
its observational distribution alone
\citep{spirtes2000causation,peters2017elements}; exact identification requires
additional restrictions, such as non-Gaussian noise or equal error variances
\citep{shimizu2006linear,peters2014identifiability}.
Interventional data can distinguish observationally equivalent DAGs. Classical
non-adaptive designs select all interventions in advance and study worst-case
requirements, bounded intervention sizes, intervention costs, or fixed budgets
\citep{eberhardt2005number,hu2014randomized,shanmugam2015learning,
kocaoglu2017cost,ghassami2018budgeted,lindgren2018experimental,
sussex2021near}. Much of this literature assumes that the observational
essential graph is known and that each intervention reveals its orientation
information without sampling error, effectively requiring population-level
observational and interventional information. Methods based on a fixed finite
interventional dataset \citep{brouillard2020differentiable} also do not adapt
subsequent interventions to the accumulated evidence.

\paragraph{Adaptive causal discovery.}
Adaptive methods select interventions using the outcomes of previous
experiments. Existing work develops strategies based on equivalence classes,
trees, clique decompositions, and adaptive graph search
\citep{he2008active,hauser2014two,greenewald2019sample,
squires2020active,choo2022verification,choo2023adaptivity}.
Most graph-theoretic methods rely on an accurately known observational
equivalence class and population-level interventional information. Recent
finite-sample methods include AIT and GIT
\citep{scherrer2021learning,olko2023trust}, as well as Bayesian approaches such
as ABCD, CBED, and ABCI, which select interventions using posterior utilities
or expected information gain
\citep{agrawal2019budgeted,tigas2022interventions,toth2022active}.
These methods generally operate under a prescribed budget. Although some
provide guarantees for the experimental-design objective, they do not establish
frequentist fixed-confidence stopping guarantees for joint graph and edge-weight
recovery. The most closely related work is \citet{elahi2024adaptive}, which proposes a
track-and-stop algorithm using finitely many interventional samples. However,
it assumes access to infinitely many observational samples, considers only
finite-state discrete variables, and targets only DAG structure recovery. In
contrast, our method starts without historical observational data and jointly
learns the graph, edge weights, and noise variances in a continuous
linear-Gaussian SEM.

\section{Problem Formulation}
\label{sec:problem}

\paragraph{Notation.}
For a positive integer $p$, let $[p]:=\{1,\ldots,p\}$.
For a matrix $\bM$, we write
$\|\bM\|_{\max}:=\max_{j,k}|M_{jk}|$.
We use $\KL(P\|Q)$ for the Kullback--Leibler (KL) divergence and
$\kl(x,y):=x\log(x/y)+(1-x)\log((1-x)/(1-y))$
for the binary relative entropy.
For a measurable space $\mathcal X$, $\mathcal P(\mathcal X)$ denotes
the set of probability measures on $\mathcal X$.

\subsection{Linear Gaussian Causal Model and Interventions}
\label{subsec:model}

We consider a linear Gaussian SEM over $p$ observed
variables,
\begin{equation}
    \bX=\bA\bX+\be,
    \qquad
    \be\sim\mathcal N(\bm 0,\bSigma),
    \label{eq:sem}
\end{equation}
parameterized by $\theta=(D,\bA,\bSigma)\in\Theta$. Here, $D$ is a directed
acyclic graph (DAG) on $[p]$, $A_{jk}\neq 0$ if and only if $k\to j$ is an
edge in $D$, and
$\bSigma=\diag(\sigma_1^2,\ldots,\sigma_p^2)$ is the diagonal noise covariance.
Thus, $A_{jk}$ represents the direct causal effect of node $k$ on node $j$.

For every node $j$, the learner may perform a perfect single-node intervention
$\operatorname{do}(X_j=u)$. Such an intervention replaces the structural
equation of node $j$ by $X_j=u$, thereby removing all incoming edges to $j$
while leaving the remaining structural equations unchanged. The intervention
value can be chosen from a known interval $[U_{j,1},U_{j,2}]$. The available
action space is therefore
\begin{equation}
    \cA_\infty
    :=
    \{0\}
    \cup
    \bigl\{(j,u):j\in[p],\ u\in[U_{j,1},U_{j,2}]\bigr\},
    \label{eq:continuous_action_space}
\end{equation}
where $a=0$ denotes passive observation and $a=(j,u)$ denotes
$\operatorname{do}(X_j=u)$. At this stage, $\cA_\infty$ is generally
uncountable; Section~\ref{sec:method} will show that restricting each interval
to its two endpoints entails no loss of information.

Let $P_a^\theta$ denote the non-degenerate Gaussian distribution observed
under intervention $a$, with $a=0$ corresponding to passive observation.
For $a=(j,u)$, let $R=[p]\setminus\{j\}$. Since the intervention removes all
incoming edges to node $j$ and fixes $X_j=u$,
\begin{equation}
    \bX_R
    =
    (\bI-\bA_{R,R})^{-1}\bA_{R,j}u
    +
    (\bI-\bA_{R,R})^{-1}\be_R.
    \label{eq:interventional_sem}
\end{equation}
The distributions under passive observation and single-node intervention are
therefore summarized as
\begin{equation}
P_a^\theta=
\begin{cases}
\mathcal N\!\left(
    \bm 0,\,
    (\bI-\bA)^{-1}\bSigma(\bI-\bA)^{-\mathsf T}
\right),
& a=0,\\[2mm]
\mathcal N\!\left(
    (\bI-\bA_{R,R})^{-1}\bA_{R,j}u,\,
    (\bI-\bA_{R,R})^{-1}\bSigma_{R,R}
    (\bI-\bA_{R,R})^{-\mathsf T}
\right),
& a=(j,u).
\end{cases}
\label{eq:interventional_distribution}
\end{equation}
In particular, the interventional mean is linear in $u$, whereas the
covariance is independent of $u$. This structure will allow us to reduce the
continuous intervention space to finitely many interventions.

\begin{assumption}[Bounded parameter class]
\label{ass:parameter_class}
The number of nodes $p$ is fixed, and every candidate graph is a DAG. There
exist known constants $\beta_{\min},A_{\max},\sigma_{\min},
\sigma_{\max},C_{\max}>0$ such that
\[
    \beta_{\min}\leq |A_{jk}|\leq A_{\max}
    \quad\text{for every nonzero }A_{jk},\quad
    \sigma_{\min}^2\leq \sigma_j^2\leq\sigma_{\max}^2,
    \quad
    \max_{j,u\in[U_{j,1},U_{j,2}]}|u|\leq C_{\max}.
\]
\end{assumption}

Since there are finitely many DAGs on $[p]$ and the parameter constraints in
Assumption~\ref{ass:parameter_class} are closed and bounded, the resulting
parameter space $\Theta$ is compact. The edge-strength lower bound separates
distinct graph supports, while the variance and magnitude bounds prevent
degenerate distributions and provide uniform control of the action-wise KL
divergences. 

\subsection{Fixed-Confidence Online Causal Learning}
\label{subsec:fixed_confidence}

Let $\theta^\star=(D^\star,\bA^\star,\bSigma^\star)\in\Theta$ denote the
unknown data-generating parameter. At round $t$, the learner selects an intervention action
$a_t\in\cA_\infty$ based on the past and obtains an independent observation
from $P_{a_t}^{\theta^\star}$. Let
\[
    \cF_t
    :=
    \sigma(a_1,\bX_1,\ldots,a_t,\bX_t)
\]
be the resulting filtration, where $\bX_t$ denotes the non-degenerate random
coordinates observed under $a_t$.
A fixed-confidence causal learning algorithm is a triple
\[
    \mathfrak A=(\pi,\tau_\delta,\widehat\theta).
\]
Here, $\pi=(\pi_t)_{t\geq1}$ is an adaptive intervention policy, where
$\pi_t$ selects the next intervention action $a_t$ based on the history available in
$\cF_{t-1}$;
$\tau_\delta$ is an $(\cF_t)$-stopping time at which the algorithm terminates,
indicating that its current estimate has reached the prescribed confidence
level; and
$\widehat\theta=(\widehat\theta_t)_{t\geq1}$ is an adapted estimator sequence,
where $\widehat\theta_t
    =
    (\widehat D_t,\widehat{\bA}_t,\widehat{\bSigma}_t)$.
The terminal output is $\widehat\theta_{\tau_\delta}$.

\begin{definition}[$(\epsilon,\delta)$-correctness]
\label{def:correctness}
Fix an accuracy level $0<\epsilon\leq\beta_{\min}/2$ and a confidence level
$\delta\in(0,1)$. An algorithm $\mathfrak A$ is
\emph{$(\epsilon,\delta)$-correct} over $\Theta$ if, for every
$\theta^\star\in\Theta$,
\begin{equation}
    \Pp_{\theta^\star}\!\left(
        \tau_\delta<\infty,\ 
        \widehat D_{\tau_\delta}=D^\star,\ 
        \|\widehat{\bA}_{\tau_\delta}-\bA^\star\|_{\max}<\epsilon
    \right)
    \geq 1-\delta.
    \label{eq:correctness}
\end{equation}
\end{definition}

Although $\bSigma^\star$ is not included in the reported accuracy criterion,
it must be estimated because both the likelihood and the KL-based sampling
and stopping rules depend on it. Our objective is to construct an
$(\epsilon,\delta)$-correct algorithm with the smallest possible expected
stopping time.

\subsection{Fundamental Information-Theoretic Limit}
\label{subsec:lower_bound}

To characterize the intrinsic difficulty of the problem, write
$\theta=(D^\theta,\bA^\theta,\bSigma^\theta)$ and
$\lambda=(D^\lambda,\bA^\lambda,\bSigma^\lambda)$. For $r>0$, define the
alternative set
\begin{equation}
    \Alt_r(\theta)
    :=
    \left\{
        \lambda\in\Theta:
        D^\lambda\neq D^\theta
        \ \text{or}\
        \|\bA^\lambda-\bA^\theta\|_{\max}\geq r
    \right\}.
    \label{eq:alternative_set}
\end{equation}
Thus, $\Alt_r(\theta)$ contains all parameters that are either structurally
different from $\theta$ or at least $r$ away from it in edge weights.

Let $\mathcal P(\cA_\infty)$ denote the set of probability measures over the
intervention space $\cA_\infty$. We define the instance-dependent information
complexity as
\begin{equation}
    C_r(\theta)
    :=
    \sup_{w\in\mathcal P(\cA_\infty)}
    \inf_{\lambda\in\Alt_r(\theta)}
    \int_{\cA_\infty}
        \KL(P_a^\theta\|P_a^\lambda)\,w(\mathrm da).
    \label{eq:information_complexity}
\end{equation}
It is the largest information rate that an intervention design can guarantee
against its least distinguishable alternative.

\begin{theorem}[Information-theoretic lower bound]
\label{thm:lower_bound}
Let $\delta\in(0,1/2)$. Every algorithm that is
$(\epsilon,\delta)$-correct uniformly over $\Theta$ satisfies
\begin{equation}
    \bbE_{\theta^\star}[\tau_\delta]
    \geq
    \frac{\kl(1-\delta,\delta)}
         {C_{2\epsilon}(\theta^\star)}.
    \label{eq:stopping_lower_bound}
\end{equation}
Consequently,
\begin{equation}
    \liminf_{\delta\downarrow0}
    \frac{\bbE_{\theta^\star}[\tau_\delta]}
         {\log(1/\delta)}
    \geq
    \frac{1}{C_{2\epsilon}(\theta^\star)}.
    \label{eq:asymptotic_lower_bound}
\end{equation}
\end{theorem}

The radius $2\epsilon$ arises because two parameters separated by at least
$2\epsilon$ cannot share an $\epsilon$-accurate estimate. Thus,
$C_{2\epsilon}(\theta^\star)$ quantifies how quickly the best possible
intervention strategy can distinguish the true model from every incompatible
alternative. Theorem~\ref{thm:lower_bound} provides the benchmark that guides
our algorithm design: we seek an $(\epsilon,\delta)$-correct procedure whose
expected stopping time approaches this fundamental limit as closely as
possible.

\paragraph{Our goal.}
We aim to develop an $(\epsilon,\delta)$-correct algorithm whose upper bound
on $\bbE_{\theta^\star}[\tau_\delta]$ is as close as possible to the
information-theoretic lower bound in
Theorem~\ref{thm:lower_bound}. We do not assume access to any historical or
pre-collected observational data; all observational and interventional
samples are acquired sequentially by the algorithm.

\section{FOCUS: Adaptive Fixed-Confidence Causal Learning}
\label{sec:method}

The lower bound in Theorem~\ref{thm:lower_bound} suggests allocating
interventions to maximize the information collected against the least
distinguishable alternative. This optimal allocation depends on the unknown
parameter $\theta^\star$ and is therefore unavailable to the learner. We now
construct \textsc{FOCUS}, which sequentially estimates the causal model,
learns an informative intervention allocation, tracks this allocation through
actual interventions, and stops once sufficient evidence has accumulated.
\subsection{Reducing the Intervention Space to Endpoints}
\label{subsec:endpoint_reduction}

The intervention space $\cA_\infty$ in
Eq.~(\ref{eq:continuous_action_space}) is generally uncountable. The following
result shows that interior intervention values provide no advantage in the
information-design problem.

\begin{proposition}[Endpoint reduction]
\label{prop:endpoint_reduction}
For every $\theta\in\Theta$ and $r>0$, the information complexity
$C_r(\theta)$ in Eq.~(\ref{eq:information_complexity}) admits an optimal
intervention design supported on passive observation and the two endpoints of
each intervention interval.
\end{proposition}

For any fixed intervention target and any pair of parameters, the Gaussian KL
divergence is a convex quadratic function of the intervention value.
Therefore, probability mass assigned to an interior value can be redistributed
to the two endpoints without decreasing the information against any
alternative. Consequently, an optimal intervention design can always be
chosen to use only the endpoints. 

We henceforth restrict attention to the finite intervention set
\begin{equation}
    \cA
    :=
    \{0\}
    \cup
    \bigl\{(j,U_{j,s}):j\in[p],\ s\in\{1,2\}\bigr\},
    \qquad
    K:=|\cA|=2p+1.
    \label{eq:finite_intervention_set}
\end{equation}
Accordingly, the information complexity in
Eq.~(\ref{eq:information_complexity}) has the equivalent finite-dimensional
form
\begin{equation}
    C_r(\theta)
    =
    \sup_{\boldsymbol{\alpha}\in\Delta(\cA)}
    \inf_{\lambda\in\Alt_r(\theta)}
    \sum_{a\in\cA}
        \alpha_a\KL(P_a^\theta\|P_a^\lambda),
    \label{eq:finite_information_complexity}
\end{equation}
where
$\Delta(\cA):=\{\boldsymbol{\alpha}\in\mathbb R_+^K:
\sum_{a\in\cA}\alpha_a=1\}$. Hence, the lower bound in
Theorem~\ref{thm:lower_bound} remains unchanged after replacing the continuous
intervention space by $\cA$.

Proposition~\ref{prop:endpoint_reduction} is essential because it converts
intervention selection over an uncountable space into an exactly equivalent
problem involving only $2p+1$ interventions. This reduction preserves the information complexity and
enables the finite-dimensional AdaHedge and tracking procedures developed
below. Throughout the remainder of the paper, all intervention designs,
intervention counts, and sums over interventions are defined with respect to
$\cA$ rather than $\cA_\infty$.
\subsection{Joint Estimation and Alternative Best Response}
\label{subsec:estimation_best_response}

Given the observations collected by time $t$, define the adaptive
log-likelihood and joint maximum-likelihood estimator (MLE) by
\begin{equation}
    \ell_t(\theta)
    :=
    \sum_{s=1}^t\log p_{a_s}^{\theta}(\bX_s),
    \qquad
    \widehat\theta_t
    \in
    \argmax_{\theta\in\Theta}\ell_t(\theta),
    \label{eq:joint_mle}
\end{equation}
where $p_a^\theta$ is the density of $P_a^\theta$. The estimator jointly
learns the graph, edge weights, and noise variances from all observational
and interventional samples:
$\widehat\theta_t=(\widehat D_t,\widehat{\bA}_t,
\widehat{\bSigma}_t)$.

If $\theta^\star$ were known, an optimizer
$\boldsymbol{\alpha}^\star$ of the max--min problem defining
$C_\epsilon(\theta^\star)$ in
Eq.~(\ref{eq:finite_information_complexity}) would specify the most
informative intervention allocation. Since both $\theta^\star$ and
$\boldsymbol{\alpha}^\star$ are unknown, we replace $\theta^\star$ with the
current estimate $\widehat\theta_t$ and learn the intervention allocation
sequentially. Given the current allocation
$\boldsymbol{\alpha}_t\in\Delta(\cA)$, an alternative-model best response is
\begin{equation}
    \lambda_t
    \in
    \argmin_{\lambda\in\Alt_\epsilon(\widehat\theta_t)}
    \sum_{a\in\cA}
        \alpha_{t,a}
        \KL(P_a^{\widehat\theta_t}\|P_a^\lambda).
    \label{eq:alternative_best_response}
\end{equation}
Here, $\alpha_{t,a}$ is the current estimated weight assigned to intervention
$a$, rather than its realized sampling proportion. As detailed in
Section~\ref{subsec:allocation_tracking}, these weights are updated online by
AdaHedge using the accumulated KL gains and are subsequently converted into
actual interventions through cumulative tracking.

The resulting gain from intervention $a$ is
\begin{equation}
    r_{t,a}
    :=
    \KL(P_a^{\widehat\theta_t}\|P_a^{\lambda_t}),
    \qquad a\in\cA.
    \label{eq:intervention_gain}
\end{equation}
Thus, $r_{t,a}$ measures how informative intervention $a$ is for
distinguishing the current estimate from its least distinguishable
$\epsilon$-alternative.

\subsection{Learning and Tracking the Intervention Allocation}
\label{subsec:allocation_tracking}

Directly solving the max--min problem in
Eq.~(\ref{eq:finite_information_complexity}) at every round can be
computationally expensive. Following the game-based approach to
fixed-confidence pure exploration, we instead let an online learner update
the intervention allocation, while the alternative player returns the best
response in Eq.~(\ref{eq:alternative_best_response})
\citep{garivier2016optimal,degenne2019nonasymptotic,elahi2024adaptive}.

We use the gain-form AdaHedge algorithm
\citep{derooij2014follow}. Initialize $\alpha_{1,a}=1/K$,
$\Delta_0=0$, and update
\begin{align}
    \alpha_{t+1,a}
    &=
    \frac{\alpha_{t,a}\exp(\eta_t r_{t,a})}
    {\sum_{b\in\cA}\alpha_{t,b}\exp(\eta_t r_{t,b})},
    \label{eq:adahedge_allocation}\\
    m_t
    &=
    \frac{1}{\eta_t}
    \log\!\left(
        \sum_{a\in\cA}\alpha_{t,a}\exp(\eta_t r_{t,a})
    \right),
    \qquad
    h_t
    =
    \sum_{a\in\cA}\alpha_{t,a}r_{t,a},
    \label{eq:adahedge_mixability}\\
    \Delta_t
    &=
    \Delta_{t-1}+m_t-h_t,
    \qquad
    \eta_{t+1}
    =
    \frac{\log K}{\Delta_t}.
    \label{eq:adahedge_learning_rate}
\end{align}
The initialization $\eta_1=+\infty$ and zero-gap cases are interpreted by
continuity. AdaHedge adapts its learning rate to the observed gains and
approaches the best fixed intervention allocation in hindsight without
requiring prior knowledge of the gain scale.

The allocation $\boldsymbol{\alpha}_t$ is fractional, whereas only one
intervention can be implemented at each round. Let 
$N_t(a):=\sum_{s=1}^t\mathbf{1}\{a_s=a\}$ 
be the number of times intervention $a$ has been implemented. Cumulative
tracking aims to maintain
$N_t(a)
    \approx
    \sum_{s=1}^t\alpha_{s,a}$, 
rather than matching $N_t(a)$ to the single-round weight $\alpha_{t,a}$.
Combined with forced exploration, the next intervention is selected by
\begin{equation}
    a_{t+1}\in
    \begin{cases}
        \displaystyle
        \argmin_{a\in\cA}N_t(a),
        & \displaystyle
          \text{if }\min_{a\in\cA}N_t(a)<\sqrt{t},\\[3mm]
        \displaystyle
        \argmax_{a\in\cA}
        \frac{\sum_{s=1}^t\alpha_{s,a}}{N_t(a)},
        & \text{otherwise}.
    \end{cases}
    \label{eq:intervention_tracking}
\end{equation}
Forced exploration ensures that every intervention is implemented infinitely
often, while cumulative tracking translates the learned fractional
allocation into actual interventions.

\subsection{Stopping Rule and the Complete Algorithm}
\label{subsec:stopping_rule}

To evaluate whether the current estimate is sufficiently separated from all
incorrect alternatives, define the accumulated information statistic
\begin{equation}
    d_t
    :=
    \inf_{\lambda\in\Alt_\epsilon(\widehat\theta_t)}
    \sum_{a\in\cA}
        N_t(a)
        \KL(P_a^{\widehat\theta_t}\|P_a^\lambda).
    \label{eq:stopping_statistic}
\end{equation}
The statistic is small when some $\epsilon$-alternative remains difficult to
distinguish and increases as evidence accumulates against every such
alternative.

The stopping rule is built on the following concentration result for the
accumulated KL divergence between the joint MLE and the true causal model.

\begin{lemma}[KL concentration]
\label{lem:kl_concentration}
Under Assumption~\ref{ass:parameter_class}, 
for every deterministic $t\geq1$ and $x\geq0$,
\begin{equation}
    \Pp_{\theta^\star}(\sum_{a\in\cA}N_t(a)\KL(P_a^{\widehat\theta_t}\|P_a^{\theta^\star})\geq x)
    \leq f_t(x),
    \label{eq:kl_concentration}
\end{equation}
where the tail function is
\begin{equation}
    f_t(x)
    :=
    \min\left\{
        1,\,
        \bigl[e^q\cN(\zeta_t(x))+1\bigr]e^{-\kappa_B x}
    \right\}.
    \label{eq:concentration_function}
\end{equation}
In particular,
\[
    f_t(x)
    =
    \operatorname{poly}\!\bigl(t,x,\log(t+1)\bigr)
    e^{-\kappa_Bx}.
\]
The quantities $q$, $\kappa_B$, $\cN(\cdot)$, and $\zeta_t(\cdot)$ are
explicit and can be computed from $p$ and the five known constants in
Assumption~\ref{ass:parameter_class}. Their specific expressions are
provided in Appendix~\ref{app:concentration_constants}.
\end{lemma}

Lemma~\ref{lem:kl_concentration} is a key technical contribution of this
work. Existing fixed-confidence methods typically establish stopping
validity through likelihood-ratio or generalized likelihood-ratio
concentration
\citep{garivier2016optimal,degenne2019nonasymptotic}. In contrast, our result
directly controls the accumulated KL divergence involving a joint estimator
of the shared causal parameters. Establishing this KL concentration requires
a new argument tailored to online causal estimation.

Using the tail function in Eq.~(\ref{eq:concentration_function}),
\textsc{FOCUS} stops whenever
\begin{equation}
    \boxed{
        \frac{\pi^2t^2}{6}\,f_t(d_t)<\delta
    }
    \label{eq:focus_stopping_rule}
\end{equation}
and returns $\widehat\theta_t$. The complete procedure is summarized in Algorithm~\ref{alg:focus}.

\begin{algorithm}[t]
\caption{\textsc{FOCUS}: Fixed-Confidence Online Causal Learning}
\label{alg:focus}
\begin{algorithmic}[1]
\Require Accuracy $\epsilon$, confidence $\delta$, intervention intervals,
and parameter class $\Theta$
\State Construct $\cA$ according to Eq.~(\ref{eq:finite_intervention_set})
\State Initialize $N_0(a)\gets0$ and $\alpha_{1,a}\gets1/K$ for every
$a\in\cA$
\State Initialize the remaining AdaHedge variables with $\Delta_0\gets0$
and $\eta_1\gets+\infty$
\For{$t=1,2,\ldots$}
    \State Select intervention $a_t$ according to
    Eq.~(\ref{eq:intervention_tracking})
    \State Implement $a_t$ and observe
    $\bX_t\sim P_{a_t}^{\theta^\star}$
    \ForAll{$a\in\cA$}
        \State $N_t(a)\gets
        N_{t-1}(a)+\mathbf{1}\{a_t=a\}$
    \EndFor
    \State Compute the joint MLE $\widehat\theta_t$ according to
    Eq.~(\ref{eq:joint_mle})
    \State Compute $d_t$ according to Eq.~(\ref{eq:stopping_statistic})
    \If{$\frac{\pi^2t^2}{6}f_t(d_t)<\delta$}
        \State \Return $\widehat\theta_t$
    \EndIf
    \State Compute $\lambda_t$ according to
    Eq.~(\ref{eq:alternative_best_response})
    \State Compute $\{r_{t,a}\}_{a\in\cA}$ according to
    Eq.~(\ref{eq:intervention_gain})
    \State Update $\boldsymbol{\alpha}_{t+1}$ and the AdaHedge variables
    according to
    Eqs.~(\ref{eq:adahedge_allocation})--(\ref{eq:adahedge_learning_rate})
\EndFor
\end{algorithmic}
\end{algorithm}

\begin{remark}[Oracle-based formulation]
\label{rem:oracles}
The theoretical version of \textsc{FOCUS} assumes access to two optimization
oracles: a joint-MLE oracle for Eq.~(\ref{eq:joint_mle}) and an
alternative-model best-response oracle for
Eqs.~(\ref{eq:alternative_best_response}) and
(\ref{eq:stopping_statistic}). Such oracle-based formulations are standard
in related fixed-confidence pure-exploration and online causal-learning
problems
\citep{degenne2019nonasymptotic,elahi2024adaptive}.
Appendix~\ref{app:oracles} provides concrete finite procedures that can replace
both oracle calls.
\end{remark}
\section{Theoretical Guarantees}
\label{sec:theory}

We now establish the fixed-confidence correctness and sample-complexity
guarantees of \textsc{FOCUS}.

\begin{theorem}[Fixed-confidence correctness]
\label{thm:focus_correctness}
Suppose Assumption~\ref{ass:parameter_class} holds and the exact joint-MLE and
alternative-model best-response procedures are used. Then, for every
$\theta^\star\in\Theta$, $0<\epsilon\leq\beta_{\min}/2$, and
$\delta\in(0,1)$, Algorithm~\ref{alg:focus} satisfies
\begin{equation}
    \Pp_{\theta^\star}\!\left(
        \tau_\delta<\infty,\ 
        \widehat D_{\tau_\delta}=D^\star,\ 
        \|\widehat{\bA}_{\tau_\delta}-\bA^\star\|_{\max}<\epsilon
    \right)
    \geq 1-\delta.
    \label{eq:focus_correctness}
\end{equation}
Therefore, \textsc{FOCUS} is $(\epsilon,\delta)$-correct in the sense of
Definition~\ref{def:correctness}.
\end{theorem}

Theorem~\ref{thm:focus_correctness} guarantees that \textsc{FOCUS}
simultaneously identifies the true causal graph and estimates every edge
weight within the prescribed accuracy. This guarantee holds despite the
adaptive choice of interventions and requires no historical observational
data: all information used for estimation and stopping is collected online.
Although $\bSigma^\star$ is not part of the reported accuracy criterion, it is
jointly estimated to evaluate the likelihoods, KL divergences, and stopping
statistic.

\begin{theorem}[Expected stopping-time upper bound]
\label{thm:upper_bound}
Suppose Assumption~\ref{ass:parameter_class} holds and the exact joint-MLE and
alternative-model best-response procedures are used. Then
Algorithm~\ref{alg:focus} terminates almost surely,
\begin{equation}
    \Pp_{\theta^\star}(\tau_\delta<\infty)=1,
    \label{eq:almost_sure_termination}
\end{equation}
and its expected stopping time satisfies
\begin{equation}
    \limsup_{\delta\downarrow0}
    \frac{\bbE_{\theta^\star}[\tau_\delta]}
         {\log(1/\delta)}
    \leq
    \frac{1}{\kappa_B C_\epsilon(\theta^\star)}.
    \label{eq:expected_time_upper_bound}
\end{equation}
\end{theorem}

Theorem~\ref{thm:upper_bound} shows that the expected number of online samples
grows at most logarithmically with the inverse confidence level. The
information complexity $C_\epsilon(\theta^\star)$ captures the intrinsic
distinguishability of the true model: a larger value means that informative
interventions separate $\theta^\star$ from its closest alternatives more
quickly. The factor $\kappa_B$ is introduced by the direct KL concentration
inequality in Lemma~\ref{lem:kl_concentration}. Together, AdaHedge, forced
exploration, and cumulative tracking ensure that the stopping statistic
accumulates information at a linear rate, leading to the upper bound in
Eq.~(\ref{eq:expected_time_upper_bound}).

\begin{remark}[Comparison with the lower bound]
\label{rem:order_optimality}
Combining Theorem~\ref{thm:lower_bound} with
Theorem~\ref{thm:upper_bound} yields
\begin{equation}
    \frac{1}{C_{2\epsilon}(\theta^\star)}
    \leq
    \liminf_{\delta\downarrow0}
    \frac{\bbE_{\theta^\star}[\tau_\delta]}
         {\log(1/\delta)}
    \leq
    \limsup_{\delta\downarrow0}
    \frac{\bbE_{\theta^\star}[\tau_\delta]}
         {\log(1/\delta)}
    \leq
    \frac{1}{\kappa_B C_\epsilon(\theta^\star)}.
    \label{eq:lower_upper_comparison}
\end{equation}
Thus, \textsc{FOCUS} is $(\epsilon,\delta)$-correct and order-optimal in the
confidence level: for fixed $\theta^\star$ and $\epsilon$, both its expected
stopping time and the information-theoretic lower bound scale as
$\Theta(\log(1/\delta))$. The remaining gap is the instance-dependent
multiplicative factor
\[
    \frac{C_{2\epsilon}(\theta^\star)}
         {\kappa_B C_\epsilon(\theta^\star)},
\]
which is independent of $\delta$.
\end{remark}

\section{Experiments}
\label{sec:experiments}

We evaluate whether \textsc{FOCUS} can efficiently recover both the causal
structure and the associated edge weights through sequentially selected
interventions. We further examine whether its empirical stopping time exhibits
the logarithmic dependence on the confidence level predicted by
Theorem~\ref{thm:upper_bound}.

\paragraph{Baselines.}
We compare \textsc{FOCUS} with the following three methods.
\textsc{RANDOM} uses the same joint estimator and intervention set as
\textsc{FOCUS}, but selects interventions uniformly at random, thereby isolating
the benefit of the proposed adaptive allocation strategy.
\textsc{AIT-DCDI} combines active intervention targeting (AIT)
\citep{scherrer2021learning} with a linear-Gaussian implementation of
differentiable causal discovery from interventional data (DCDI)
\citep{brouillard2020differentiable}.
Finally, \textsc{CBED} is a causal Bayesian experimental-design method that
selects informative intervention targets and values using expected information
gain \citep{tigas2022interventions}. We adapt all methods to the same collection
of observational and single-node perfect interventions and give them the same
online sample budget. For a fair comparison of parameter estimation, the edge
weights of each estimated graph are fitted using all samples collected by the
corresponding method.

\paragraph{Simulation settings.}
For each Monte Carlo repetition, we first sample a random topological ordering
of $p$ nodes and generate an Erd\H{o}s--R\'enyi DAG by independently including
each admissible directed edge with probability $\rho$. Conditional on the
resulting graph, the nonzero edge weights are independently sampled from
$\operatorname{Unif}([-\!A_{\max},-\beta_{\min}]
\cup[\beta_{\min},A_{\max}])$, and the diagonal entries of
$\bSigma^\star$ are independently sampled from
$\operatorname{Unif}([\sigma_{\min}^2,\sigma_{\max}^2])$. Throughout the
experiments, we set $\beta_{\min}=0.15$, $A_{\max}=1.5$,
$\sigma_{\min}^2=0.8$, and $\sigma_{\max}^2=1.2$. Samples are then drawn from
the linear-Gaussian SEM $\bX=\bA^\star\bX+\boldsymbol{\varepsilon}$, where
$\boldsymbol{\varepsilon}\sim\mathcal N(\boldsymbol{0},\bSigma^\star)$.
At each online round, an algorithm selects either the observational action or
one of the two endpoint interventions on a single node, as defined in
Eq.~(\ref{eq:finite_intervention_set}), and receives one new sample. The
intervention interval is $[-2,2]$ for every node, so the two endpoint
interventions set the selected variable to either $-2$ or $2$. For FOCUS, we use the
target estimation accuracy $\epsilon=0.07\leq \beta_{min}/2$. Within each repetition, all four
methods are evaluated on the same underlying DAG, coefficients, noise
variances, and sample budget.

We consider five graph settings: $(p,\rho)\in
    \{(5,0.3),(6,0.3),(7,0.3),(6,0.4),(6,0.5)\}$.The representative setting $(p,\rho)=(6,0.3)$ is reported in the main text. For
each setting, we conduct $1000$ independent Monte Carlo repetitions, regenerating
the DAG and its parameters in every repetition. The curves report the sample
means, and the shaded regions show the corresponding 95\% confidence intervals. The remaining experimental results are provided in
Appendix~\ref{app:experiments} due to space constraints.

\begin{figure*}[t]
    \centering
    \begin{subfigure}[t]{0.32\textwidth}
        \centering
        \includegraphics[width=\linewidth]{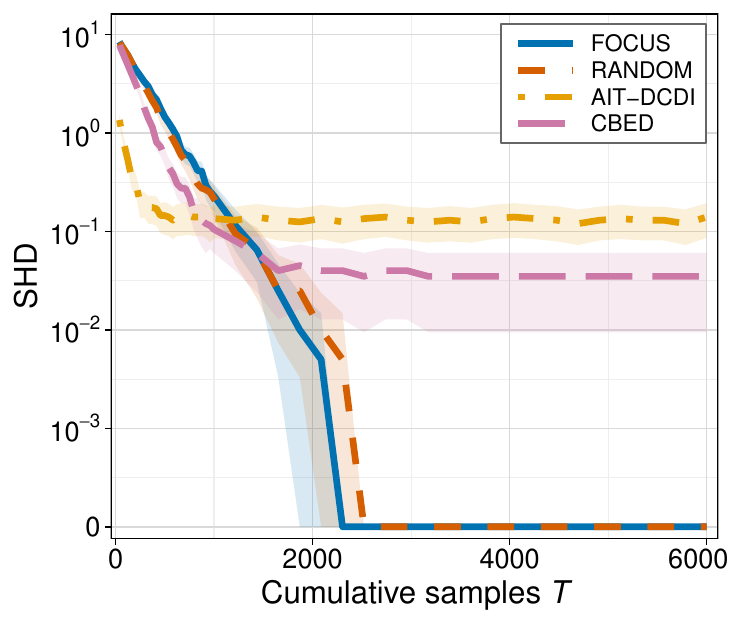}
        \caption{Structure error}
        \label{fig:main_experiment_shd}
    \end{subfigure}
    \hfill
    \begin{subfigure}[t]{0.32\textwidth}
        \centering
        \includegraphics[width=\linewidth]{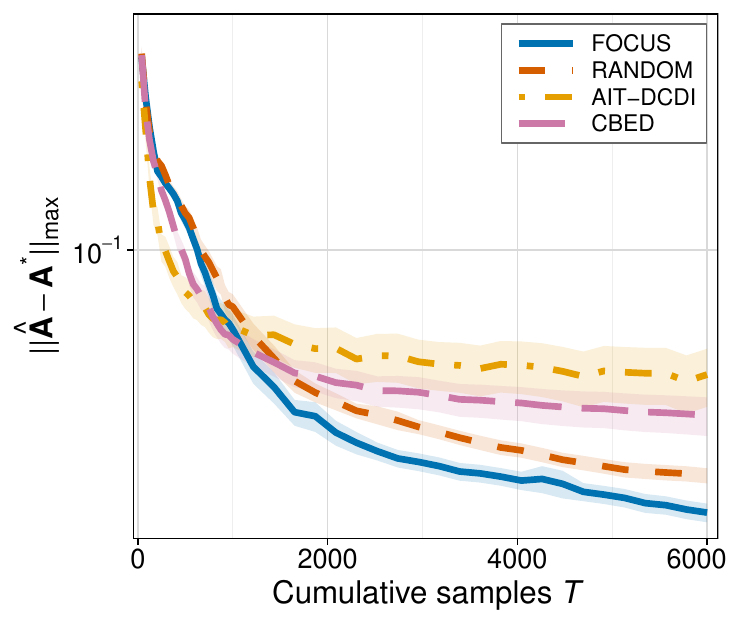}
        \caption{Edge-weight error}
        \label{fig:main_experiment_weight}
    \end{subfigure}
    \hfill
    \begin{subfigure}[t]{0.32\textwidth}
        \centering
        \includegraphics[width=\linewidth]{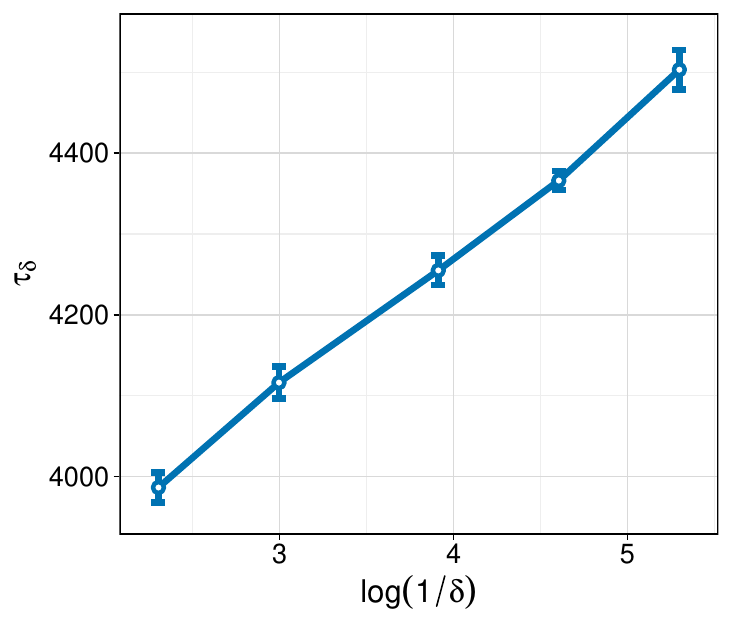}
        \caption{Stopping-time scaling}
        \label{fig:main_experiment_stopping}
    \end{subfigure}

    \caption{
        Performance for $p=6$ and $\rho=0.3$. Panels (a)--(c) show SHD, maximum
edge-weight estimation error, and the stopping time of \textsc{FOCUS} versus
$\log(1/\delta)$, respectively. Shaded bands and error bars denote 95\%
confidence intervals.
    }
    \label{fig:main_experiment}
\end{figure*}

\paragraph{Evaluation metrics.}
Each experiment contains three panels. The first reports the structural Hamming
distance $\operatorname{SHD}(\widehat D_n,D^\star)$ against the number $n$ of collected samples, where a smaller value indicates
more accurate causal-structure recovery. The second reports the maximum
edge-weight error $\|\widehat{\bA}_n-\bA^\star\|_{\max}
    :=
    \max_{i,j}
    \bigl|
        \widehat A_{n,ij}-A_{ij}^\star
    \bigr|$, which evaluates parameter recovery in addition to graph identification.
Both metrics are reported for all four methods under the same sample budget.
The third panel concerns the fixed-confidence behavior of \textsc{FOCUS} and
plots its empirical expected stopping time
$\tau_\delta$ against $\log(1/\delta)$ for
confidence levels $\delta \in \{0.1,0.05,0.02,0.01,0.005\}$.

\paragraph{Results.}
Figure~\ref{fig:main_experiment} presents the representative setting with
$p=6$ and $\rho=0.3$. As shown in
Figures~\ref{fig:main_experiment_shd} and
\ref{fig:main_experiment_weight}, although \textsc{FOCUS} improves slightly
more slowly during the initial rounds, it soon overtakes the competing methods
and achieves a clear advantage in both structural and edge-weight recovery.
Its improvement over \textsc{RANDOM} highlights the benefit of learning a
nonuniform intervention allocation rather than distributing samples uniformly
across actions, while its advantage over \textsc{AIT-DCDI} and \textsc{CBED}
shows the value of designing interventions for joint graph and edge-weight
recovery. Comparing these two figures also shows that SHD often reaches
zero before $\|\widehat{\bA}_n-\bA^\star\|_{\max}$ stabilizes, indicating that
graph recovery alone may stop before the edge weights are accurately estimated
and motivating the joint target of \textsc{FOCUS}. Finally,
Figure~\ref{fig:main_experiment_stopping} shows an approximately linear
relationship between $\tau_\delta$ and $\log(1/\delta)$, consistent with
Theorem~\ref{thm:upper_bound}. Together, the three panels indicate that, when
\textsc{FOCUS} terminates and returns its estimate, it already substantially
outperforms the other methods under comparable sample budgets.

\section{Conclusion}
\label{sec:conclusion}

We studied fully online fixed-confidence recovery of causal structure and edge
weights without any historical observational data. Starting from zero samples,
\textsc{FOCUS} adaptively selects interventions and determines when sufficient
evidence has been collected. It jointly updates the causal model and
intervention allocation using only sequentially collected data. We established
an instance-dependent lower bound, proved that \textsc{FOCUS} is
$(\epsilon,\delta)$-correct, and showed that it achieves the optimal
$\Theta(\log(1/\delta))$ dependence up to an instance-dependent constant.
Experiments demonstrate improved structural and parameter recovery and confirm
the predicted stopping-time trend. They also show that edge-weight estimation
remains meaningful after the SHD reaches zero. Future work includes extending the framework to nonlinear models and latent
confounding, and more general intervention mechanisms under practical
experimental constraints.

\bibliography{ref}
\bibliographystyle{iclr2027_conference}

\clearpage
\newpage
\appendix

\startcontents[appendices]

\section*{Appendix Contents}
\printcontents[appendices]{}{1}{\setcounter{tocdepth}{1}}

\paragraph{Additional notation.}
For a vector $\bx$, $\|\bx\|_2$ denotes its Euclidean norm.
For a matrix $\bM$, $\|\bM\|_F$ and $\|\bM\|_{op}$ denote its
Frobenius and operator norms, respectively.
For a DAG $D$, $\pa_D(j)$ denotes the parent set of node $j$.

\section{Proof of Proposition~\ref{prop:endpoint_reduction}}
\label{app:endpoint_reduction}

\begin{proof}
Fix $\theta,\lambda\in\Theta$ and an intervention target $j\in[p]$. Let
$R=[p]\setminus\{j\}$ and define
\[
    \bg_j^\theta
    :=
    (\bI-\bA_{R,R}^\theta)^{-1}\bA_{R,j}^\theta,
    \qquad
    \bV_j^\theta
    :=
    (\bI-\bA_{R,R}^\theta)^{-1}
    \bSigma_{R,R}^\theta
    (\bI-\bA_{R,R}^\theta)^{-\mathsf T}.
\]
Because $\bA^\theta$ corresponds to a DAG, $\bI-\bA_{R,R}^\theta$ is
invertible. Under the intervention $\operatorname{do}(X_j=u)$, the
non-degenerate coordinates therefore follow
\[
    P_{j,u}^\theta
    =
    \mathcal N(\bg_j^\theta u,\bV_j^\theta).
\]
The covariance does not depend on $u$, while the mean is linear in $u$.

For two Gaussian distributions, the KL divergence gives
\begin{align}
    \KL(P_{j,u}^\theta\|P_{j,u}^\lambda)
    &=
    c_j(\theta,\lambda)
    +
    \frac{u^2}{2}
    (\bg_j^\lambda-\bg_j^\theta)^{\mathsf T}
    (\bV_j^\lambda)^{-1}
    (\bg_j^\lambda-\bg_j^\theta),
    \label{eq:endpoint_kl_quadratic}
\end{align}
where
\[
    c_j(\theta,\lambda)
    :=
    \frac{1}{2}
    \left[
        \tr\!\left\{
            (\bV_j^\lambda)^{-1}\bV_j^\theta
        \right\}
        -(p-1)
        +
        \log\frac{\det(\bV_j^\lambda)}
                  {\det(\bV_j^\theta)}
    \right]
\]
is independent of $u$. The coefficient of $u^2$ in
Eq.~(\ref{eq:endpoint_kl_quadratic}) is nonnegative because
$\bV_j^\lambda$ is positive definite. Hence
\[
    u\longmapsto
    \KL(P_{j,u}^\theta\|P_{j,u}^\lambda)
\]
is convex on $[U_{j,1},U_{j,2}]$ for every pair
$\theta,\lambda\in\Theta$.

Suppose first that $U_{j,1}<U_{j,2}$. Every
$u\in[U_{j,1},U_{j,2}]$ can be written as
\[
    u
    =
    \rho_j(u)U_{j,1}
    +
    \bigl(1-\rho_j(u)\bigr)U_{j,2},
    \qquad
    \rho_j(u)
    :=
    \frac{U_{j,2}-u}{U_{j,2}-U_{j,1}}.
\]
By convexity,
\begin{align}
    \KL(P_{j,u}^\theta\|P_{j,u}^\lambda)
    \leq{}&
    \rho_j(u)
    \KL(P_{j,U_{j,1}}^\theta\|P_{j,U_{j,1}}^\lambda)
    \nonumber\\
    &+
    \bigl(1-\rho_j(u)\bigr)
    \KL(P_{j,U_{j,2}}^\theta\|P_{j,U_{j,2}}^\lambda).
    \label{eq:endpoint_convexity}
\end{align}
If $U_{j,1}=U_{j,2}$, there is only one admissible intervention value and
the conclusion is immediate.

Now take an arbitrary probability measure
$w\in\mathcal P(\cA_\infty)$. We construct a new measure
$\widetilde w$ supported on $\cA$ by retaining the mass of the passive
action and transferring the mass of every interior intervention to its two
endpoints. Specifically, let $w_j$ be the restriction of $w$ to interventions
on node $j$, identified with a measure on $[U_{j,1},U_{j,2}]$, and set
\begin{align*}
    \widetilde w(\{0\})
    &:=
    w(\{0\}),\\
    \widetilde w(\{(j,U_{j,1})\})
    &:=
    \int_{U_{j,1}}^{U_{j,2}}
        \rho_j(u)\,w_j(\mathrm du),\\
    \widetilde w(\{(j,U_{j,2})\})
    &:=
    \int_{U_{j,1}}^{U_{j,2}}
        \bigl(1-\rho_j(u)\bigr)\,w_j(\mathrm du).
\end{align*}
The transferred masses are nonnegative and preserve the total mass assigned
to interventions on each node. Thus, $\widetilde w$ is a probability measure
supported on the finite set $\cA$.

Integrating Eq.~(\ref{eq:endpoint_convexity}) with respect to $w_j$ and
summing over all intervention targets shows that, for every
$\lambda\in\Theta$,
\begin{equation}
    \int_{\cA_\infty}
        \KL(P_a^\theta\|P_a^\lambda)\,w(\mathrm da)
    \leq
    \sum_{a\in\cA}
        \widetilde w(a)
        \KL(P_a^\theta\|P_a^\lambda).
    \label{eq:endpoint_design_domination}
\end{equation}
Importantly, the construction of $\widetilde w$ depends only on $w$ and the
intervention intervals, not on the alternative parameter $\lambda$.
Therefore, taking the infimum over
$\lambda\in\Alt_r(\theta)$ preserves the inequality:
\[
    \inf_{\lambda\in\Alt_r(\theta)}
    \int_{\cA_\infty}
        \KL(P_a^\theta\|P_a^\lambda)\,w(\mathrm da)
    \leq
    \inf_{\lambda\in\Alt_r(\theta)}
    \sum_{a\in\cA}
        \widetilde w(a)
        \KL(P_a^\theta\|P_a^\lambda).
\]
Hence every design over the continuous action space is weakly dominated by
an endpoint-supported design. Since $\cA\subseteq\cA_\infty$, the reverse
inequality between the corresponding optimal values is immediate, and thus
\[
    C_r(\theta)
    =
    \sup_{\boldsymbol{\alpha}\in\Delta(\cA)}
    \inf_{\lambda\in\Alt_r(\theta)}
    \sum_{a\in\cA}
        \alpha_a\KL(P_a^\theta\|P_a^\lambda).
\]

Finally, $\Delta(\cA)$ is compact, $\Alt_r(\theta)$ is compact, and the
action-wise KL divergences are continuous in their parameters. Therefore,
the displayed max--min objective attains its supremum on
$\Delta(\cA)$. An optimal design supported only on the passive action and the
two endpoint interventions for each node consequently exists.
\end{proof}

\section{Proof of Theorem~\ref{thm:lower_bound}}
\label{app:lower_bound}

We first state the change-of-measure inequality used in the proof. Its
finite-action version is Lemma~1 of
\citet{kaufmann2016complexity}. The following formulation extends it to a
general measurable action space by replacing the expected action counts with
an expected occupation measure.

\begin{lemma}[Change of measure for adaptive interventions]
\label{lem:change_of_measure}
Fix two parameters $\theta,\lambda\in\Theta$ and consider any adaptive
intervention policy. Let $\tau$ be an $(\cF_t)$-stopping time such that
$\bbE_\theta[\tau]<\infty$, and let $E\in\cF_\tau$. Define the expected
occupation measure under $\theta$ by
\[
    \Gamma_\tau^\theta(B)
    :=
    \bbE_\theta\!\left[
        \sum_{t=1}^{\tau}\mathbf{1}\{a_t\in B\}
    \right],
    \qquad
    B\subseteq\cA_\infty.
\]
Then
\begin{equation}
    \int_{\cA_\infty}
        \KL(P_a^\theta\|P_a^\lambda)\,
        \Gamma_\tau^\theta(\mathrm da)
    \geq
    \kl\!\left(
        \Pp_\theta(E),\Pp_\lambda(E)
    \right).
    \label{eq:change_of_measure}
\end{equation}
For a finite action space, the left-hand side reduces to
\[
    \sum_{a\in\cA}
        \bbE_\theta[N_\tau(a)]
        \KL(P_a^\theta\|P_a^\lambda).
\]
\end{lemma}

\begin{proof}
Let $p_a^\theta$ and $p_a^\lambda$ denote the densities of the observations
under action $a$. The log-likelihood ratio of the history up to time $t$ is
\[
    L_t(\theta,\lambda)
    :=
    \sum_{s=1}^t
    \log
    \frac{p_{a_s}^\theta(\bX_s)}
         {p_{a_s}^\lambda(\bX_s)}.
\]
Although the actions are selected adaptively, the intervention policy is the
same under $\theta$ and $\lambda$. Conditional on the past, the probability
with which the policy selects $a_s$ therefore appears in both likelihoods and
cancels from their ratio. Consequently, $L_\tau(\theta,\lambda)$ is the
log-likelihood ratio between the stopped histories under $\theta$ and
$\lambda$.

Taking conditional expectations under $\theta$ gives
\begin{align}
    \bbE_\theta[L_\tau(\theta,\lambda)]
    &=
    \bbE_\theta\!\left[
        \sum_{s=1}^{\tau}
        \bbE_\theta\!\left[
            \left.
            \log
            \frac{p_{a_s}^\theta(\bX_s)}
                 {p_{a_s}^\lambda(\bX_s)}
            \right|\cF_{s-1},a_s
        \right]
    \right]
    \nonumber\\
    &=
    \bbE_\theta\!\left[
        \sum_{s=1}^{\tau}
        \KL(P_{a_s}^\theta\|P_{a_s}^\lambda)
    \right]
    \nonumber\\
    &=
    \int_{\cA_\infty}
        \KL(P_a^\theta\|P_a^\lambda)\,
        \Gamma_\tau^\theta(\mathrm da).
    \label{eq:expected_stopped_likelihood_ratio}
\end{align}
The interchange is valid because $\bbE_\theta[\tau]<\infty$ and the
action-wise KL divergences are uniformly bounded over the compact parameter
and intervention spaces.

The KL divergence between the distributions of the complete stopped histories
equals $\bbE_\theta[L_\tau(\theta,\lambda)]$. Mapping a stopped history to the
binary variable $\mathbf{1}\{E\}$ and applying the data-processing inequality
for KL divergence yields
\[
    \bbE_\theta[L_\tau(\theta,\lambda)]
    \geq
    \kl\!\left(
        \Pp_\theta(E),\Pp_\lambda(E)
    \right).
\]
Combining this inequality with
Eq.~(\ref{eq:expected_stopped_likelihood_ratio}) proves the result.
\end{proof}

We now apply Lemma~\ref{lem:change_of_measure} to the fixed-confidence causal
learning problem.

\begin{proof}[Proof of Theorem~\ref{thm:lower_bound}]
If $\bbE_{\theta^\star}[\tau_\delta]=\infty$, the claimed lower bound holds
trivially. We may therefore assume
$\bbE_{\theta^\star}[\tau_\delta]<\infty$, which also implies that
$\tau_\delta<\infty$ almost surely under $\theta^\star$.

Define the event that the algorithm returns an $\epsilon$-accurate estimate
of $\theta^\star$ by
\begin{equation}
    E_{\theta^\star}
    :=
    \left\{
        \tau_\delta<\infty,\ 
        \widehat D_{\tau_\delta}=D^\star,\ 
        \|\widehat{\bA}_{\tau_\delta}-\bA^\star\|_{\max}<\epsilon
    \right\}.
    \label{eq:true_success_event}
\end{equation}
Since the algorithm is $(\epsilon,\delta)$-correct uniformly over $\Theta$,
\[
    \Pp_{\theta^\star}(E_{\theta^\star})\geq1-\delta.
\]

Fix an arbitrary
$\lambda=(D^\lambda,\bA^\lambda,\bSigma^\lambda)
\in\Alt_{2\epsilon}(\theta^\star)$.
Under $\lambda$, define its corresponding success event
\[
    E_\lambda
    :=
    \left\{
        \tau_\delta<\infty,\ 
        \widehat D_{\tau_\delta}=D^\lambda,\ 
        \|\widehat{\bA}_{\tau_\delta}-\bA^\lambda\|_{\max}<\epsilon
    \right\}.
\]
Uniform correctness also gives
$\Pp_\lambda(E_\lambda)\geq1-\delta$.

The events $E_{\theta^\star}$ and $E_\lambda$ are disjoint. Indeed, if
$D^\lambda\neq D^\star$, they require the algorithm to return two different
graphs. Otherwise, the definition of
$\Alt_{2\epsilon}(\theta^\star)$ gives
$\|\bA^\lambda-\bA^\star\|_{\max}\geq2\epsilon$. An estimate belonging to both
events would satisfy
\[
    \|\bA^\lambda-\bA^\star\|_{\max}
    \leq
    \|\bA^\lambda-\widehat{\bA}_{\tau_\delta}\|_{\max}
    +
    \|\widehat{\bA}_{\tau_\delta}-\bA^\star\|_{\max}
    <2\epsilon,
\]
which is a contradiction. Therefore,
\[
    \Pp_\lambda(E_{\theta^\star})
    \leq
    1-\Pp_\lambda(E_\lambda)
    \leq\delta.
\]
This disjointness argument explains the appearance of the radius
$2\epsilon$ in the lower bound.

Applying Lemma~\ref{lem:change_of_measure} with
$\theta=\theta^\star$, stopping time $\tau_\delta$, and event
$E_{\theta^\star}$ gives, for every
$\lambda\in\Alt_{2\epsilon}(\theta^\star)$,
\begin{equation}
    \int_{\cA_\infty}
        \KL(P_a^{\theta^\star}\|P_a^\lambda)\,
        \Gamma_{\tau_\delta}^{\theta^\star}(\mathrm da)
    \geq
    \kl(1-\delta,\delta).
    \label{eq:change_measure_lower_bound}
\end{equation}
Here we used that binary relative entropy is increasing in its first argument
and decreasing in its second argument whenever
the first argument exceeds the second, together with
$\delta<1/2$.

Normalize the expected occupation measure by defining
\[
    \overline w_\delta(B)
    :=
    \frac{
        \Gamma_{\tau_\delta}^{\theta^\star}(B)
    }{
        \bbE_{\theta^\star}[\tau_\delta]
    }.
\]
Because exactly one action is selected at every round,
$\Gamma_{\tau_\delta}^{\theta^\star}(\cA_\infty)
=\bbE_{\theta^\star}[\tau_\delta]$, so
$\overline w_\delta\in\mathcal P(\cA_\infty)$. Dividing
Eq.~(\ref{eq:change_measure_lower_bound}) by
$\bbE_{\theta^\star}[\tau_\delta]$ and taking the infimum over
$\lambda\in\Alt_{2\epsilon}(\theta^\star)$ yields
\begin{align}
    \kl(1-\delta,\delta)
    &\leq
    \bbE_{\theta^\star}[\tau_\delta]
    \inf_{\lambda\in\Alt_{2\epsilon}(\theta^\star)}
    \int_{\cA_\infty}
        \KL(P_a^{\theta^\star}\|P_a^\lambda)\,
        \overline w_\delta(\mathrm da)
    \nonumber\\
    &\leq
    \bbE_{\theta^\star}[\tau_\delta]\,
    C_{2\epsilon}(\theta^\star),
    \label{eq:information_complexity_lower_step}
\end{align}
where the final inequality follows from the definition of
$C_{2\epsilon}(\theta^\star)$ as the supremum over all intervention designs.
Rearranging gives
\[
    \bbE_{\theta^\star}[\tau_\delta]
    \geq
    \frac{\kl(1-\delta,\delta)}
         {C_{2\epsilon}(\theta^\star)}.
\]

Finally,
\[
    \kl(1-\delta,\delta)
    =
    (1-2\delta)
    \log\frac{1-\delta}{\delta},
\]
and hence
\[
    \lim_{\delta\downarrow0}
    \frac{\kl(1-\delta,\delta)}
         {\log(1/\delta)}
    =1.
\]
Dividing the finite-confidence lower bound by $\log(1/\delta)$ and taking the
limit inferior proves
\[
    \liminf_{\delta\downarrow0}
    \frac{\bbE_{\theta^\star}[\tau_\delta]}
         {\log(1/\delta)}
    \geq
    \frac{1}{C_{2\epsilon}(\theta^\star)}.
\]
\end{proof}

\section{Proof of Lemma~\ref{lem:kl_concentration}}
\label{app:concentration}

We establish several auxiliary results before proving
Lemma~\ref{lem:kl_concentration}. For each $a\in\cA$, let $J(a)$ denote the
random coordinates observed under action $a$: $J(0)=[p]$, while
$J(a)=[p]\setminus\{j\}$ if $a$ intervenes on node $j$. Write
$d_a:=|J(a)|$ and
\[
    P_a^\theta
    =
    \mathcal N(\boldsymbol\mu_a(\theta),\bV_a(\theta))
    \quad\text{on }\mathbb R^{d_a}.
\]

\subsection{Uniform bounds for the interventional Gaussian family}

Define the computable constants
\begin{equation}
    G_p
    :=
    \sum_{k=0}^{p-1}(pA_{\max})^k,
    \qquad
    B_\mu
    :=
    G_p\sqrt{p}\,A_{\max}C_{\max},
    \label{eq:uniform_gaussian_constants}
\end{equation}
and
\begin{equation}
    v_-
    :=
    \frac{\sigma_{\min}^2}{(1+pA_{\max})^2},
    \qquad
    v_+
    :=
    G_p^2\sigma_{\max}^2.
    \label{eq:uniform_covariance_constants}
\end{equation}

\begin{lemma}[Uniform Gaussian bounds]
\label{lem:uniform_gaussian_bounds}
Under Assumption~\ref{ass:parameter_class}, for every
$\theta\in\Theta$ and $a\in\cA$,
\begin{equation}
    \|\boldsymbol\mu_a(\theta)\|_2
    \leq B_\mu,
    \qquad
    v_-\bI_{d_a}
    \preceq
    \bV_a(\theta)
    \preceq
    v_+\bI_{d_a}.
    \label{eq:uniform_mean_covariance_bounds}
\end{equation}
Consequently, the action-wise KL divergences are uniformly bounded:
\[
    \sup_{a\in\cA}
    \sup_{\theta,\lambda\in\Theta}
    \KL(P_a^\theta\|P_a^\lambda)
    \leq D_{\max},
\]
where one valid computable choice is
\begin{equation}
    D_{\max}
    :=
    \frac{1}{2}
    \left[
        p\left(
            \frac{v_+}{v_-}
            -1
            +\log\frac{v_+}{v_-}
        \right)
        +
        \frac{4B_\mu^2}{v_-}
    \right].
    \label{eq:uniform_kl_bound}
\end{equation}
\end{lemma}

\begin{proof}
Because $\bA$ represents a DAG, it becomes strictly triangular after a
simultaneous permutation of its rows and columns. Hence $\bA^p=0$ and
\[
    (\bI-\bA)^{-1}
    =
    \sum_{k=0}^{p-1}\bA^k.
\]
Since $\|\bA\|_{\mathrm{op}}\leq\|\bA\|_{\mathrm F}\leq pA_{\max}$,
we obtain
$\|(\bI-\bA)^{-1}\|_{\mathrm{op}}\leq G_p$. The same argument applies
to every principal submatrix $\bA_{R,R}$.

For an intervention $a=(j,U_{j,s})$, Eq.~(\ref{eq:interventional_distribution})
gives
\[
    \boldsymbol\mu_a(\theta)
    =
    (\bI-\bA_{R,R})^{-1}\bA_{R,j}U_{j,s}.
\]
Using $\|\bA_{R,j}\|_2\leq\sqrt p A_{\max}$ and
$|U_{j,s}|\leq C_{\max}$ proves the mean bound. The observational mean is zero
and therefore satisfies the same bound.

For the upper covariance bound, use
$\|\bSigma_{R,R}\|_{\mathrm{op}}\leq\sigma_{\max}^2$ to obtain
\[
    \lambda_{\max}(\bV_a(\theta))
    \leq
    \|(\bI-\bA_{R,R})^{-1}\|_{\mathrm{op}}^2
    \sigma_{\max}^2
    \leq v_+.
\]
For the lower bound,
\[
    \lambda_{\min}(\bV_a(\theta))
    \geq
    \sigma_{\min}^2
    s_{\min}\!\left(
        (\bI-\bA_{R,R})^{-1}
    \right)^2
    =
    \frac{\sigma_{\min}^2}
         {\|\bI-\bA_{R,R}\|_{\mathrm{op}}^2}
    \geq v_-.
\]
The observational covariance is handled identically.

Finally, for $P=\mathcal N(\boldsymbol\mu,\bV)$ and
$Q=\mathcal N(\boldsymbol\nu,\bW)$ in dimension $d\leq p$,
\[
    \KL(P\|Q)
    =
    \frac12
    \left[
        \tr(\bW^{-1}\bV)-d
        +\log\frac{\det\bW}{\det\bV}
        +(\boldsymbol\nu-\boldsymbol\mu)^{\mathsf T}
         \bW^{-1}
         (\boldsymbol\nu-\boldsymbol\mu)
    \right].
\]
The bounds in Eq.~(\ref{eq:uniform_mean_covariance_bounds}) imply
$\tr(\bW^{-1}\bV)\leq d\,v_+/v_-$,
$\log(\det\bW/\det\bV)\leq d\log(v_+/v_-)$, and
\[
    (\boldsymbol\nu-\boldsymbol\mu)^{\mathsf T}
    \bW^{-1}
    (\boldsymbol\nu-\boldsymbol\mu)
    \leq
    \frac{4B_\mu^2}{v_-}.
\]
This proves Eq.~(\ref{eq:uniform_kl_bound}).
\end{proof}

\subsection{A uniform Bhattacharyya--KL comparison}

For distributions $P$ and $Q$ with densities $p$ and $q$, define their
Bhattacharyya distance by
\[
    \mathsf B(P,Q)
    :=
    -\log\int\sqrt{pq}.
\]
For Gaussian distributions
$P=\mathcal N(\boldsymbol\mu_0,\bV_0)$ and
$Q=\mathcal N(\boldsymbol\mu_1,\bV_1)$, direct Gaussian integration gives
\begin{equation}
    \mathsf B(P,Q)
    =
    \frac18
    \boldsymbol\Delta^{\mathsf T}
    \overline{\bV}^{-1}
    \boldsymbol\Delta
    +
    \frac12
    \log
    \frac{\det\overline{\bV}}
         {\sqrt{\det\bV_0\det\bV_1}},
    \label{eq:gaussian_bhattacharyya}
\end{equation}
where
$\boldsymbol\Delta=\boldsymbol\mu_1-\boldsymbol\mu_0$ and
$\overline{\bV}=(\bV_0+\bV_1)/2$.

Let $\kappa_V:=v_+/v_-$ and define
\[
    g(r)
    :=
    \frac12(r-1-\log r),
    \qquad
    h(r)
    :=
    \frac12\log\frac{1+r}{2\sqrt r}.
\]
At $r=1$, define the continuous extension
$h(1)/g(1):=1/4$. Set
\begin{equation}
    \kappa_\mu
    :=
    \frac{v_-}{4v_+},
    \qquad
    \kappa_\Sigma
    :=
    \min_{r\in[\kappa_V^{-1},\kappa_V]}
        \frac{h(r)}{g(r)},
    \qquad
    \kappa_B
    :=
    \min\{\kappa_\mu,\kappa_\Sigma\}.
    \label{eq:kappa_b_definition}
\end{equation}
The minimization defining $\kappa_\Sigma$ is one-dimensional and deterministic.
Moreover, continuity and strict positivity away from $r=1$ imply
$\kappa_B>0$.

\begin{lemma}[Bhattacharyya distance dominates reverse KL]
\label{lem:bhattacharyya_reverse_kl}
For every $a\in\cA$ and $\theta,\lambda\in\Theta$,
\begin{equation}
    \mathsf B(P_a^\theta,P_a^\lambda)
    \geq
    \kappa_B
    \KL(P_a^\lambda\|P_a^\theta).
    \label{eq:bhattacharyya_reverse_kl}
\end{equation}
\end{lemma}

\begin{proof}
We compare the mean and covariance parts separately. From
$\overline{\bV}\preceq v_+\bI$,
\[
    \mathsf B_\mu(P,Q)
    \geq
    \frac{\|\boldsymbol\Delta\|_2^2}{8v_+}.
\]
Meanwhile,
\[
    \KL_\mu(Q\|P)
    =
    \frac12
    \boldsymbol\Delta^{\mathsf T}
    \bV_0^{-1}
    \boldsymbol\Delta
    \leq
    \frac{\|\boldsymbol\Delta\|_2^2}{2v_-}.
\]
Thus,
$\mathsf B_\mu(P,Q)\geq\kappa_\mu\KL_\mu(Q\|P)$.

For the covariance terms, let $r_1,\ldots,r_d$ be the eigenvalues of
$\bV_0^{-1/2}\bV_1\bV_0^{-1/2}$. The uniform covariance bounds imply
$r_i\in[\kappa_V^{-1},\kappa_V]$. Diagonalizing the two covariance expressions
gives
\[
    \KL_\Sigma(Q\|P)
    =
    \sum_{i=1}^d g(r_i),
    \qquad
    \mathsf B_\Sigma(P,Q)
    =
    \sum_{i=1}^d h(r_i).
\]
By the definition of $\kappa_\Sigma$,
$\mathsf B_\Sigma(P,Q)\geq
\kappa_\Sigma\KL_\Sigma(Q\|P)$. Adding the mean and covariance parts and
taking their smaller constant proves the lemma.
\end{proof}

\subsection{A fixed-parameter likelihood-ratio bound}

For a fixed candidate parameter $\theta\in\Theta$, define
\begin{align}
    L_t(\theta)
    &:=
    \ell_t(\theta)-\ell_t(\theta^\star),
    \label{eq:fixed_parameter_likelihood_ratio}\\
    D_t(\theta,\theta^\star)
    &:=
    \sum_{a\in\cA}
        N_t(a)\KL(P_a^\theta\|P_a^{\theta^\star}).
    \label{eq:fixed_parameter_kl}
\end{align}

\begin{lemma}[Fixed-parameter deviation]
\label{lem:fixed_parameter_deviation}
For every deterministic $t\geq1$, fixed $\theta\in\Theta$, and $b,y\geq0$,
\begin{equation}
    \Pp_{\theta^\star}\!\left(
        L_t(\theta)\geq-b,\,
        D_t(\theta,\theta^\star)\geq y
    \right)
    \leq
    \exp\!\left(
        \frac b2-\kappa_B y
    \right).
    \label{eq:fixed_parameter_deviation}
\end{equation}
\end{lemma}

\begin{proof}
Define
\begin{equation}
    M_t^\theta
    :=
    \exp\!\left\{
        \frac12L_t(\theta)
        +
        \sum_{s=1}^t
        \mathsf B(P_{a_s}^{\theta^\star},P_{a_s}^{\theta})
    \right\}.
    \label{eq:bhattacharyya_martingale}
\end{equation}
Because $a_s$ is selected using only $\cF_{s-1}$, conditional expectation under
$\theta^\star$ gives
\begin{align*}
    &\bbE_{\theta^\star}\!\left[
        \left.
        \exp\!\left\{
            \frac12
            \log
            \frac{p_{a_s}^\theta(\bX_s)}
                 {p_{a_s}^{\theta^\star}(\bX_s)}
            +
            \mathsf B(P_{a_s}^{\theta^\star},P_{a_s}^{\theta})
        \right\}
        \right|\cF_{s-1}
    \right]\\
    &\qquad=
    \exp\!\left\{
        \mathsf B(P_{a_s}^{\theta^\star},P_{a_s}^{\theta})
    \right\}
    \int
        \sqrt{
            p_{a_s}^{\theta^\star}(x)
            p_{a_s}^{\theta}(x)
        }\,
        \mathrm dx
    =1.
\end{align*}
Hence $(M_t^\theta)_{t\geq0}$ is a nonnegative mean-one martingale.

On the event in Eq.~(\ref{eq:fixed_parameter_deviation}),
Lemma~\ref{lem:bhattacharyya_reverse_kl} implies
\[
    M_t^\theta
    \geq
    \exp\{-b/2+\kappa_B y\}.
\]
Markov's inequality and
$\bbE_{\theta^\star}[M_t^\theta]=1$ then give the desired bound.
\end{proof}

\subsection{Covering and Lipschitz bounds}

Define a metric on the induced collection of action distributions by
\begin{equation}
    d_{\mathsf G}(\theta,\lambda)
    :=
    \max_{a\in\cA}
    \left\{
        \|\boldsymbol\mu_a(\theta)
          -\boldsymbol\mu_a(\lambda)\|_2
        +
        \|\bV_a(\theta)-\bV_a(\lambda)\|_{\mathrm F}
    \right\}.
    \label{eq:gaussian_parameter_metric}
\end{equation}
Let $s_a:=d_a(d_a+1)/2$ and
\begin{equation}
    q
    :=
    \sum_{a\in\cA}(d_a+s_a)
    =
    p+\frac{p(p+1)}2
    +
    2p\left[
        (p-1)+\frac{p(p-1)}2
    \right].
    \label{eq:covering_dimension}
\end{equation}

A standard volumetric argument shows that a Euclidean ball of radius $R$ in
$\mathbb R^d$ has an $\xi$-net of cardinality at most
$(1+2R/\xi)^d$: take a maximal $\xi$-separated subset and compare the volumes
of the disjoint balls of radius $\xi/2$ around its elements with the enclosing
ball of radius $R+\xi/2$.

Applying this argument to every mean and covariance component, and retaining
one feasible parameter from each nonempty product cell, gives a
$\zeta$-net of the image of $\Theta$ under $d_{\mathsf G}$ with cardinality at
most
\begin{equation}
    \cN(\zeta)
    :=
    \prod_{a\in\cA}
    \left(
        1+\frac{8B_\mu}{\zeta}
    \right)^{d_a}
    \left(
        1+\frac{8\sqrt{d_a}\,v_+}{\zeta}
    \right)^{s_a},
    \qquad 0<\zeta\leq1.
    \label{eq:covering_number}
\end{equation}
The constant $8$ accounts for selecting feasible representatives from the
nonempty covering cells.

We also require a tail bound for the observed Gaussian vectors.

\begin{lemma}[Uniform Gaussian norm tail]
\label{lem:uniform_gaussian_norm_tail}
For $z>0$, define
\begin{equation}
    R(z)
    :=
    B_\mu
    +
    \sqrt{v_+}\bigl(\sqrt p+\sqrt{2z}\bigr).
    \label{eq:gaussian_truncation_radius}
\end{equation}
Then, for every deterministic $t\geq1$,
\begin{equation}
    \Pp_{\theta^\star}\!\left(
        \max_{1\leq s\leq t}\|\bX_s\|_2>R(z)
    \right)
    \leq
    te^{-z}.
    \label{eq:uniform_gaussian_norm_tail}
\end{equation}
\end{lemma}

\begin{proof}
Lemma~1 of \citet{laurent2000adaptive} states that if
$Z\sim\chi_d^2$, then, for every $z>0$,
\[
    \Pp\!\left(
        Z\geq d+2\sqrt{dz}+2z
    \right)
    \leq e^{-z}.
\]
In particular, for $\boldsymbol Z\sim\mathcal N(\boldsymbol0,\bI_d)$,
\[
    \Pp\!\left(
        \|\boldsymbol Z\|_2
        >
        \sqrt d+\sqrt{2z}
    \right)
    \leq e^{-z}.
\]
Conditionally on $\cF_{s-1}$ and $a_s$, write
$\bX_s=\boldsymbol\mu_{a_s}(\theta^\star)
+\bV_{a_s}(\theta^\star)^{1/2}\boldsymbol Z_s$. By
Lemma~\ref{lem:uniform_gaussian_bounds},
\[
    \|\bX_s\|_2
    \leq
    B_\mu+\sqrt{v_+}\|\boldsymbol Z_s\|_2.
\]
Since $d_{a_s}\leq p$, the conditional probability of
$\|\bX_s\|_2>R(z)$ is at most $e^{-z}$. A union bound over
$s=1,\ldots,t$ proves the claim.
\end{proof}

On the event
$\max_{s\leq t}\|\bX_s\|_2\leq R$, differentiation of the Gaussian log density
gives the per-sample Lipschitz constant
\begin{equation}
    L_\ell(R)
    :=
    \frac{R+B_\mu}{v_-}
    +
    \frac12
    \left[
        \frac{\sqrt p}{v_-}
        +
        \frac{(R+B_\mu)^2}{v_-^2}
    \right].
    \label{eq:log_likelihood_lipschitz_constant}
\end{equation}
Indeed, for a Gaussian density with mean $\boldsymbol\mu$ and covariance
$\bV$,
\[
    \nabla_{\boldsymbol\mu}\log p_{\boldsymbol\mu,\bV}(x)
    =
    \bV^{-1}(x-\boldsymbol\mu),
\]
while
\[
    \nabla_{\bV}\log p_{\boldsymbol\mu,\bV}(x)
    =
    -\frac12\bV^{-1}
    +
    \frac12
    \bV^{-1}(x-\boldsymbol\mu)(x-\boldsymbol\mu)^{\mathsf T}\bV^{-1}.
\]
The uniform mean and covariance bounds therefore imply
\[
    \|\nabla_{\boldsymbol\mu}\log p_{\boldsymbol\mu,\bV}(x)\|_2
    \leq
    \frac{R+B_\mu}{v_-},
\]
and
\[
    \|\nabla_{\bV}\log p_{\boldsymbol\mu,\bV}(x)\|_{\mathrm F}
    \leq
    \frac12
    \left[
        \frac{\sqrt p}{v_-}
        +
        \frac{(R+B_\mu)^2}{v_-^2}
    \right].
\]
The mean-value theorem consequently gives, whenever
$d_{\mathsf G}(\theta,\lambda)\leq\zeta$,
\begin{equation}
    |L_t(\theta)-L_t(\lambda)|
    \leq
    tL_\ell(R)\zeta.
    \label{eq:likelihood_lipschitz_bound}
\end{equation}

Similarly, differentiating
$\KL(P_a^\theta\|P_a^{\theta^\star})$ with respect to the mean and covariance
of its first argument gives the uniform constant
\begin{equation}
    L_{\mathrm K}
    :=
    \frac{2B_\mu+\sqrt p}{v_-}.
    \label{eq:kl_lipschitz_constant}
\end{equation}
Hence
\begin{equation}
    \left|
        D_t(\theta,\theta^\star)
        -
        D_t(\lambda,\theta^\star)
    \right|
    \leq
    tL_{\mathrm K}\zeta
    \qquad
    \text{whenever }
    d_{\mathsf G}(\theta,\lambda)\leq\zeta.
    \label{eq:kl_lipschitz_bound}
\end{equation}

\subsection{Completion of the concentration proof}

For $t\geq1$ and $x\geq0$, define
\begin{align}
    z_t(x)
    &:=
    \kappa_Bx+2\log(t+1),
    \label{eq:concentration_truncation_level}\\
    H_t(x)
    &:=
    \frac12
    L_\ell\!\left(R(z_t(x))\right)
    +
    \kappa_BL_{\mathrm K},
    \label{eq:concentration_lipschitz_level}\\
    \zeta_t(x)
    &:=
    \min\left\{
        1,\,
        \frac{q}{tH_t(x)}
    \right\}.
    \label{eq:concentration_covering_resolution}
\end{align}
All quantities in these definitions depend only on $t,x,p$ and the known
constants in Assumption~\ref{ass:parameter_class}.

\begin{proof}[Proof of Lemma~\ref{lem:kl_concentration}]
For brevity, write
\[
    K_t^\star
    :=
    \sum_{a\in\cA}
        N_t(a)
        \KL(P_a^{\widehat\theta_t}\|P_a^{\theta^\star})
    =
    D_t(\widehat\theta_t,\theta^\star).
\]
Since $\widehat\theta_t$ is a global MLE and $\theta^\star\in\Theta$,
\[
    L_t(\widehat\theta_t)
    =
    \ell_t(\widehat\theta_t)-\ell_t(\theta^\star)
    \geq0.
\]
Therefore,
\begin{equation}
    \{K_t^\star\geq x\}
    \subseteq
    \left\{
        \exists\theta\in\Theta:
        L_t(\theta)\geq0,\,
        D_t(\theta,\theta^\star)\geq x
    \right\}.
    \label{eq:mle_event_embedding}
\end{equation}

Fix $z>0$ and $\zeta\in(0,1]$, and let
$\{\theta_1,\ldots,\theta_M\}$ be a $\zeta$-net under
$d_{\mathsf G}$, with $M\leq\cN(\zeta)$. On the truncation event
$\max_{s\leq t}\|\bX_s\|_2\leq R(z)$, any witness $\theta$ in
Eq.~(\ref{eq:mle_event_embedding}) has a net representative $\theta_i$
satisfying
\[
    L_t(\theta_i)
    \geq
    -tL_\ell(R(z))\zeta,
    \qquad
    D_t(\theta_i,\theta^\star)
    \geq
    x-tL_{\mathrm K}\zeta.
\]
Applying Lemma~\ref{lem:fixed_parameter_deviation} to every net point and then
using a union bound yields
\begin{align}
    \Pp_{\theta^\star}(K_t^\star\geq x)
    \leq{}&
    \cN(\zeta)
    \exp\!\left\{
        -\kappa_Bx
        +
        t\zeta
        \left[
            \frac12L_\ell(R(z))
            +
            \kappa_BL_{\mathrm K}
        \right]
    \right\}
    +
    te^{-z}.
    \label{eq:preliminary_kl_concentration}
\end{align}
If $x-tL_{\mathrm K}\zeta<0$, the first term on the right-hand side is at
least one, so the same inequality remains trivially valid.

We now choose $z=z_t(x)$ and $\zeta=\zeta_t(x)$. By construction,
\[
    te^{-z_t(x)}
    =
    \frac{t}{(t+1)^2}e^{-\kappa_Bx}
    \leq
    e^{-\kappa_Bx},
\]
and
$t\zeta_t(x)H_t(x)\leq q$. Substituting these choices into
Eq.~(\ref{eq:preliminary_kl_concentration}) gives
\[
    \Pp_{\theta^\star}(K_t^\star\geq x)
    \leq
    \left[
        e^q\cN(\zeta_t(x))+1
    \right]
    e^{-\kappa_Bx}.
\]
Since every probability is at most one, we conclude that
\[
    \Pp_{\theta^\star}(K_t^\star\geq x)
    \leq
    f_t(x),
\]
where
\[
    f_t(x)
    :=
    \min\left\{
        1,\,
        \left[
            e^q\cN(\zeta_t(x))+1
        \right]
        e^{-\kappa_Bx}
    \right\}.
\]

Finally, Eq.~(\ref{eq:covering_number}) implies
$\cN(\zeta)\leq C_{\mathrm{cov}}\zeta^{-q}$ for an explicit constant
$C_{\mathrm{cov}}$. Moreover,
$H_t(x)=O(1+x+\log(t+1))$, and hence
$\zeta_t(x)^{-q}$ is polynomial in
$t,x$, and $\log(t+1)$. Therefore,
\[
    f_t(x)
    =
    \operatorname{poly}\!\bigl(t,x,\log(t+1)\bigr)
    e^{-\kappa_Bx},
\]
up to truncation at one, as claimed.
\end{proof}

\subsection{Explicit computation of the concentration function}
\label{app:concentration_constants}

We now give an explicit construction of every quantity appearing in
$f_t(x)$. All the constants below are deterministic and depend only on $p$
and the known bounds in Assumption~\ref{ass:parameter_class}; in particular,
they do not depend on $\theta^\star$ or on the observations.

First, define
\begin{equation}
    G_p
    :=
    \sum_{k=0}^{p-1}(pA_{\max})^k,
    \qquad
    B_\mu
    :=
    G_p\sqrt{p}\,A_{\max}C_{\max}.
    \label{eq:explicit_mean_constants}
\end{equation}
Because every candidate graph is acyclic, after a suitable permutation its
weighted adjacency matrix is nilpotent. Hence
$(\bI-\bA)^{-1}=\sum_{k=0}^{p-1}\bA^k$, and $G_p$ uniformly bounds the
operator norm of this inverse. Consequently, the mean vector under every
action satisfies
\[
    \|\boldsymbol{\mu}_a(\theta)\|_2\leq B_\mu,
    \qquad
    a\in\cA,\quad \theta\in\Theta.
\]

Uniform lower and upper bounds on all action-wise covariance matrices are
given by
\begin{equation}
    v_-
    :=
    \frac{\sigma_{\min}^2}
         {(1+pA_{\max})^2},
    \qquad
    v_+
    :=
    G_p^2\sigma_{\max}^2.
    \label{eq:explicit_covariance_constants}
\end{equation}
Thus, for every $a\in\cA$ and $\theta\in\Theta$,
\[
    v_-\bI
    \preceq
    \bV_a(\theta)
    \preceq
    v_+\bI.
\]
Let
\[
    \kappa_V:=\frac{v_+}{v_-}.
\]
To define the exponential rate in the concentration inequality, introduce
\[
    g(r):=\frac{1}{2}\bigl(r-1-\log r\bigr),
    \qquad
    h(r):=\frac{1}{2}
    \log\left(\frac{1+r}{2\sqrt r}\right).
\]
At $r=1$, the ratio $h(r)/g(r)$ is defined by continuity and equals $1/4$.
Set
\begin{equation}
    \kappa_\mu
    :=
    \frac{v_-}{4v_+},
    \qquad
    \kappa_\Sigma
    :=
    \min_{r\in[\kappa_V^{-1},\kappa_V]}
    \frac{h(r)}{g(r)},
    \qquad
    \kappa_B
    :=
    \min\{\kappa_\mu,\kappa_\Sigma\}.
    \label{eq:explicit_kappa_B}
\end{equation}
The interval in Eq.~(\ref{eq:explicit_kappa_B}) is compact and the extended
ratio is continuous and strictly positive. Therefore, $\kappa_B>0$ and can
be computed numerically from $v_-$ and $v_+$.

It remains to specify the covering term. Let $d_a$ denote the dimension of
the nondegenerate random coordinates under action $a$. Thus,
\[
    d_a=
    \begin{cases}
        p, & a=0,\\
        p-1, & a=(j,u),
    \end{cases}
    \qquad
    s_a:=\frac{d_a(d_a+1)}{2}.
\]
Since there is one observational action and two endpoint interventions for
each node, we may take the covering dimension to be
\begin{equation}
    q
    :=
    \sum_{a\in\cA}(d_a+s_a)
    =
    p+\frac{p(p+1)}{2}
    +
    2p\left(
        p-1+\frac{p(p-1)}{2}
    \right).
    \label{eq:explicit_covering_dimension}
\end{equation}
A valid explicit upper bound on the required covering number is
\begin{equation}
    \cN(\zeta)
    :=
    \prod_{a\in\cA}
    \left(1+\frac{8B_\mu}{\zeta}\right)^{d_a}
    \left(
        1+\frac{8\sqrt{d_a}\,v_+}{\zeta}
    \right)^{s_a},
    \qquad \zeta>0.
    \label{eq:explicit_covering_number}
\end{equation}

For $z\geq0$, define the uniform sample-radius bound
\begin{equation}
    R(z)
    :=
    B_\mu+
    \sqrt{v_+}\bigl(\sqrt p+\sqrt{2z}\bigr).
    \label{eq:explicit_sample_radius}
\end{equation}
The two Lipschitz constants used in the covering argument can then be chosen
as
\begin{align}
    L_{\ell}(R)
    &:=
    \frac{R+B_\mu}{v_-}
    +
    \frac{1}{2}
    \left\{
        \frac{\sqrt p}{v_-}
        +
        \frac{(R+B_\mu)^2}{v_-^2}
    \right\},
    \label{eq:explicit_likelihood_lipschitz}\\
    L_{\mathrm{KL}}
    &:=
    \frac{2B_\mu+\sqrt p}{v_-}.
    \label{eq:explicit_kl_lipschitz}
\end{align}
For each $t\geq1$ and $x\geq0$, let
\begin{align}
    z_t(x)
    &:=
    \kappa_Bx+2\log(t+1),
    \label{eq:explicit_z_tx}\\
    H_t(x)
    &:=
    \frac{1}{2}
    L_{\ell}\!\left(R(z_t(x))\right)
    +
    \kappa_B L_{\mathrm{KL}},
    \label{eq:explicit_H_tx}\\
    \zeta_t(x)
    &:=
    \min\left\{
        1,\,
        \frac{q}{tH_t(x)}
    \right\}.
    \label{eq:explicit_zeta_tx}
\end{align}
Substituting Eqs.~(\ref{eq:explicit_kappa_B}),
(\ref{eq:explicit_covering_number}), and
(\ref{eq:explicit_zeta_tx}) into
Eq.~(\ref{eq:concentration_function}) gives the fully computable tail
function
\begin{equation}
    f_t(x)
    =
    \min\left\{
        1,\,
        \left[
            e^q
            \prod_{a\in\cA}
            \left(1+\frac{8B_\mu}{\zeta_t(x)}\right)^{d_a}
            \left(
                1+
                \frac{8\sqrt{d_a}\,v_+}{\zeta_t(x)}
            \right)^{s_a}
            +1
        \right]
        e^{-\kappa_Bx}
    \right\}.
    \label{eq:fully_explicit_concentration_function}
\end{equation}

Therefore, $f_t(x)$ can be evaluated using only
$p$, $A_{\max}$, $\sigma_{\min}$, $\sigma_{\max}$, and $C_{\max}$.
The remaining constant $\beta_{\min}$ in
Assumption~\ref{ass:parameter_class} is needed to separate different graph
supports, but it is not required for this uniform Gaussian concentration
bound. Hence every quantity in $f_t$ is known before the online experiment
begins.

\section{Proof of Theorem~\ref{thm:focus_correctness}}
\label{app:correctness}

\begin{proof}
Define
\[
    K_t^\star
    :=
    \sum_{a\in\cA}
        N_t(a)
        \KL\!\left(
            P_a^{\widehat\theta_t}
            \middle\|
            P_a^{\theta^\star}
        \right).
\]
Suppose that \textsc{FOCUS} stops at time $t$ but returns an incorrect
estimate, namely,
\[
    \widehat D_t\neq D^\star
    \quad\text{or}\quad
    \|\widehat{\bA}_t-\bA^\star\|_{\max}\geq\epsilon.
\]
By the definition of $\Alt_\epsilon(\widehat\theta_t)$, this implies
$\theta^\star\in\Alt_\epsilon(\widehat\theta_t)$. Therefore,
Eq.~(\ref{eq:stopping_statistic}) gives
\begin{equation}
    d_t
    \leq
    \sum_{a\in\cA}
        N_t(a)
        \KL\!\left(
            P_a^{\widehat\theta_t}
            \middle\|
            P_a^{\theta^\star}
        \right)
    =
    K_t^\star.
    \label{eq:incorrect_stop_kl}
\end{equation}

Moreover, stopping at time $t$ implies
\[
    f_t(d_t)<\frac{6\delta}{\pi^2t^2}.
\]
Combining this inequality with Eq.~(\ref{eq:incorrect_stop_kl}) and the
fixed-time concentration bound in Lemma~\ref{lem:kl_concentration}, the
nestedness of the upper-tail events $\{K_t^\star\geq x\}$ yields
\begin{equation}
    \Pp_{\theta^\star}\!\left(
        \tau_\delta=t,\,
        \widehat D_t\neq D^\star
        \ \text{or}\
        \|\widehat{\bA}_t-\bA^\star\|_{\max}\geq\epsilon
    \right)
    \leq
    \frac{6\delta}{\pi^2t^2}.
    \label{eq:incorrect_stop_at_t}
\end{equation}
This argument applies directly to the stopping condition
$f_t(d_t)<6\delta/(\pi^2t^2)$.

Taking a union bound over all possible stopping times gives
\[
\begin{aligned}
    &\Pp_{\theta^\star}\!\left(
        \tau_\delta<\infty,\,
        \widehat D_{\tau_\delta}\neq D^\star
        \ \text{or}\
        \|\widehat{\bA}_{\tau_\delta}-\bA^\star\|_{\max}
        \geq\epsilon
    \right)                                               
    &\leq
    \sum_{t=1}^{\infty}\frac{6\delta}{\pi^2t^2}
    =\delta.
\end{aligned}
\]
Finally, Theorem~\ref{thm:upper_bound} guarantees that
$\Pp_{\theta^\star}(\tau_\delta<\infty)=1$. Hence,
\[
    \Pp_{\theta^\star}\!\left(
        \tau_\delta<\infty,\,
        \widehat D_{\tau_\delta}=D^\star,\,
        \|\widehat{\bA}_{\tau_\delta}-\bA^\star\|_{\max}<\epsilon
    \right)
    \geq1-\delta.
\]
Thus, \textsc{FOCUS} is $(\epsilon,\delta)$-correct.
\end{proof}

\section{Proof of Theorem~\ref{thm:upper_bound}}
\label{app:upper_bound}

This section proves the almost-sure termination and expected stopping-time
upper bound of \textsc{FOCUS}. We first collect the tracking and online-learning
properties used in the proof, then establish a finite expected localization
time for the joint MLE, and finally show that the stopping statistic grows
linearly.

For notational convenience, write
\[
    k_a(\theta,\lambda)
    :=
    \KL(P_a^\theta\|P_a^\lambda),
    \qquad a\in\cA.
\]
Compactness of $\Theta$, continuity of the action-wise Gaussian distributions,
and the uniform covariance lower bound imply that
\begin{equation}
    B_{\mathrm{KL}}
    :=
    \max_{a\in\cA}
    \sup_{\theta,\lambda\in\Theta}
    k_a(\theta,\lambda)
    <\infty.
    \label{eq:uniform_one_step_kl_bound}
\end{equation}
Thus, the gain vectors in Eq.~(\ref{eq:intervention_gain}) satisfy
$0\leq r_{t,a}\leq B_{\mathrm{KL}}$ uniformly over $t$ and $a$.

As usual in the analysis of stopping algorithms, we consider the infinite
continuation of the sampling rule: if the stopping condition is met, the
algorithm may be regarded as continuing to generate actions, estimates, and
gain vectors according to the same update rules. This continuation does not
change the actual stopping time.

\subsection{Tracking and AdaHedge guarantees}
\label{app:upper_bound_preliminaries}

We first state the cumulative-tracking property used below.

\begin{lemma}[Forced exploration and cumulative tracking]
\label{lem:cumulative_tracking}
Under the intervention-selection rule in
Eq.~(\ref{eq:intervention_tracking}), for every $t\geq K$ and $a\in\cA$,
\begin{equation}
    \sum_{s=1}^t\alpha_{s,a}
    -(K-1)(\sqrt t+2)
    \leq
    N_t(a)
    \leq
    \max\left\{
        1+\sum_{s=1}^t\alpha_{s,a},
        \sqrt t+1
    \right\}.
    \label{eq:cumulative_tracking_bound}
\end{equation}
Moreover, there exist deterministic constants
$c_{\mathrm{fe}}>0$ and $t_{\mathrm{fe}}<\infty$, depending only on $K$,
such that
\begin{equation}
    \min_{a\in\cA}N_t(a)
    \geq c_{\mathrm{fe}}\sqrt t,
    \qquad t\geq t_{\mathrm{fe}}.
    \label{eq:forced_exploration_lower_bound}
\end{equation}
\end{lemma}

The bound in Eq.~(\ref{eq:cumulative_tracking_bound}) is precisely the
allocation-matching result of Lemma~12 in
\citet{elahi2024adaptive}, specialized to the action set $\cA$.
The lower bound in Eq.~(\ref{eq:forced_exploration_lower_bound}) is a direct
consequence of the forced-exploration branch: whenever the least-sampled
action falls below the threshold $\sqrt t$, one of the least-sampled actions
is selected. Since there are only $K$ actions and the threshold increases by
$o(1)$ per round, the minimum count can lag behind $\sqrt t$ by at most an
additive constant depending on $K$. Reducing the multiplicative constant if
necessary gives Eq.~(\ref{eq:forced_exploration_lower_bound}).

We next record the regret guarantee for the gain-form AdaHedge procedure.

\begin{lemma}[AdaHedge regret]
\label{lem:adahedge_regret}
Suppose $0\leq r_{t,a}\leq B_{\mathrm{KL}}$ for every $t$ and $a$. Define
\begin{equation}
    \mathcal R_T^{\mathrm{AH}}
    :=
    \max_{a\in\cA}
    \sum_{t=1}^T r_{t,a}
    -
    \sum_{t=1}^T
    \sum_{a\in\cA}\alpha_{t,a}r_{t,a}.
    \label{eq:adahedge_regret_definition}
\end{equation}
Then
\begin{equation}
    \mathcal R_T^{\mathrm{AH}}
    \leq
    B_{\mathrm{KL}}\sqrt{T\log K}
    +
    B_{\mathrm{KL}}
    \left(
        \frac{4}{3}\log K+2
    \right).
    \label{eq:adahedge_regret_bound}
\end{equation}
In particular, $\mathcal R_T^{\mathrm{AH}}=o(T)$.
Furthermore, for every integer $0\leq T_0<T$,
\begin{equation}
    \sum_{t=T_0+1}^T
    \sum_{a\in\cA}\alpha_{t,a}r_{t,a}
    \geq
    \max_{a\in\cA}
    \sum_{t=T_0+1}^T r_{t,a}
    -
    \mathcal R_T^{\mathrm{AH}}
    -
    T_0B_{\mathrm{KL}}.
    \label{eq:truncated_adahedge_bound}
\end{equation}
\end{lemma}

\begin{proof}
The first inequality follows from Corollary~9 of
\citet{derooij2014follow}, applied to the loss vectors
$-r_t$ and rescaled from unit range to range $B_{\mathrm{KL}}$.
Indeed, the per-round range is at most $B_{\mathrm{KL}}$, so the
range-dependent bound in that corollary gives
Eq.~(\ref{eq:adahedge_regret_bound}).

For the truncated bound, the definition of
$\mathcal R_T^{\mathrm{AH}}$ gives
\[
    \sum_{t=1}^T
    \sum_{a\in\cA}\alpha_{t,a}r_{t,a}
    \geq
    \max_{a\in\cA}\sum_{t=1}^T r_{t,a}
    -
    \mathcal R_T^{\mathrm{AH}}.
\]
The first $T_0$ mixed gains are nonnegative, while for each fixed action
their sum over these rounds is at most $T_0B_{\mathrm{KL}}$. Removing the
first $T_0$ rounds therefore yields
Eq.~(\ref{eq:truncated_adahedge_bound}).
\end{proof}

We will also use the mixed-strategy form of the information game.

\begin{lemma}[Mixed-alternative representation]
\label{lem:mixed_alternative_representation}
For every $r>0$,
\begin{equation}
    C_r(\theta^\star)
    =
    \inf_{\nu\in\mathcal P(\Alt_r(\theta^\star))}
    \max_{a\in\cA}
    \int_{\Alt_r(\theta^\star)}
        k_a(\theta^\star,\lambda)\,
        \nu(\mathrm d\lambda),
    \label{eq:mixed_alternative_representation}
\end{equation}
where $\mathcal P(\Alt_r(\theta^\star))$ denotes the set of probability
measures on $\Alt_r(\theta^\star)$.
\end{lemma}

This is the continuous-alternative analogue of the standard max--min
duality used in fixed-confidence pure exploration
\citep{degenne2019nonasymptotic}; see also Lemma~9 of
\citet{elahi2024adaptive}. It follows by allowing the alternative player to
randomize and applying minimax duality to the compact allocation simplex and
the convex set of probability measures over the compact alternative set.

\subsection{Finite expected localization time}
\label{app:finite_localization}

We measure the distance between two causal parameters by
\begin{equation}
    \mathsf d_\Theta(\theta,\lambda)
    :=
    \max\left\{
        \mathbf 1\{D^\theta\neq D^\lambda\},
        \|\bA^\theta-\bA^\lambda\|_{\max},
        \|\bSigma^\theta-\bSigma^\lambda\|_{\max}
    \right\}.
    \label{eq:parameter_distance_upper_proof}
\end{equation}
For a fixed localization radius $\xi>0$, define the last-exit time
\begin{equation}
    \tau_{\mathrm{loc}}(\xi)
    :=
    \inf\left\{
        n\geq1:
        \mathsf d_\Theta(\widehat\theta_t,\theta^\star)<\xi
        \text{ for every }t\geq n
    \right\}.
    \label{eq:localization_time}
\end{equation}
This random time is introduced only for the analysis and is not observed by
the algorithm. Notice that $\xi$ is unrelated to the AdaHedge learning rate
$\eta_t$.

Define the parameter-separation constant
\begin{equation}
    \Delta_\Theta(\xi;\theta^\star)
    :=
    \inf_{\substack{\theta\in\Theta:\\
        \mathsf d_\Theta(\theta,\theta^\star)\geq\xi}}
    \sum_{a\in\cA}
        k_a(\theta,\theta^\star).
    \label{eq:parameter_separation_constant}
\end{equation}
Compactness of $\Theta$, continuity of the Gaussian KL divergence, and
identifiability of the finite interventional family imply
\begin{equation}
    \Delta_\Theta(\xi;\theta^\star)>0
    \qquad\text{for every fixed }\xi>0.
    \label{eq:positive_parameter_separation}
\end{equation}

\begin{lemma}[Finite expected localization time]
\label{lem:finite_expected_localization}
For every fixed $\xi>0$,
\begin{equation}
    \bbE_{\theta^\star}
    [\tau_{\mathrm{loc}}(\xi)]
    <\infty.
    \label{eq:finite_expected_localization}
\end{equation}
\end{lemma}

\begin{proof}
Recall the accumulated divergence between the joint MLE and the true
parameter,
\[
    K_t^\star
    =
    \sum_{a\in\cA}
        N_t(a)k_a(\widehat\theta_t,\theta^\star).
\]
For $t\geq t_{\mathrm{fe}}$, if
$\mathsf d_\Theta(\widehat\theta_t,\theta^\star)\geq\xi$, then
Eqs.~(\ref{eq:forced_exploration_lower_bound}) and
(\ref{eq:parameter_separation_constant}) give
\begin{align}
    K_t^\star
    &\geq
    \min_{a\in\cA}N_t(a)
    \sum_{a\in\cA}
        k_a(\widehat\theta_t,\theta^\star) \notag\\
    &\geq
    c_{\mathrm{fe}}\sqrt t\,
    \Delta_\Theta(\xi;\theta^\star).
    \label{eq:bad_estimate_implies_large_kl}
\end{align}
Therefore, Lemma~\ref{lem:kl_concentration} implies
\begin{equation}
    \Pp_{\theta^\star}\!\left(
        \mathsf d_\Theta(\widehat\theta_t,\theta^\star)\geq\xi
    \right)
    \leq
    f_t\!\left(
        c_{\mathrm{fe}}\sqrt t\,
        \Delta_\Theta(\xi;\theta^\star)
    \right).
    \label{eq:localization_probability_bound}
\end{equation}

The event $\{\tau_{\mathrm{loc}}(\xi)>n\}$ means that the estimator leaves the
$\xi$-neighborhood at least once after time $n$. Hence
\[
    \Pp_{\theta^\star}(\tau_{\mathrm{loc}}(\xi)>n)
    \leq
    \sum_{t=n}^{\infty}
    \Pp_{\theta^\star}\!\left(
        \mathsf d_\Theta(\widehat\theta_t,\theta^\star)\geq\xi
    \right).
\]
Using the tail-sum formula and
Eq.~(\ref{eq:localization_probability_bound}), we obtain
\begin{align}
    \bbE_{\theta^\star}[\tau_{\mathrm{loc}}(\xi)]
    &=
    \sum_{n=0}^{\infty}
    \Pp_{\theta^\star}(\tau_{\mathrm{loc}}(\xi)>n) \notag\\
    &\leq
    t_{\mathrm{fe}}
    +
    \sum_{t=t_{\mathrm{fe}}}^{\infty}
    (t+1)
    f_t\!\left(
        c_{\mathrm{fe}}\sqrt t\,
        \Delta_\Theta(\xi;\theta^\star)
    \right).
    \label{eq:localization_expectation_bound}
\end{align}
By the explicit expression derived in
Appendix~\ref{app:concentration_constants}, the summand in
Eq.~(\ref{eq:localization_expectation_bound}) is bounded by a polynomial in
$t$ times $\exp(-c\sqrt t)$ for some $c>0$. The series is therefore finite.
\end{proof}

In particular, $\tau_{\mathrm{loc}}(\xi)<\infty$ almost surely.

\subsection{Continuity error terms}
\label{app:upper_bound_continuity}

We introduce two deterministic continuity moduli. The first controls the
effect of replacing the true first argument of a KL divergence by a nearby
estimated parameter:
\begin{equation}
    u(\xi)
    :=
    \max_{a\in\cA}
    \sup_{\lambda\in\Theta}
    \sup_{\substack{\theta\in\Theta:\\
        \mathsf d_\Theta(\theta,\theta^\star)\leq\xi}}
    \left|
        k_a(\theta,\lambda)
        -
        k_a(\theta^\star,\lambda)
    \right|.
    \label{eq:first_kl_continuity_modulus}
\end{equation}

To define the second modulus, for
$\boldsymbol{\alpha}\in\Delta(\cA)$ let
\begin{equation}
    b_\epsilon(\boldsymbol{\alpha},\theta)
    \in
    \argmin_{\lambda\in\Alt_\epsilon(\theta)}
    \sum_{a\in\cA}
        \alpha_a k_a(\theta,\lambda)
    \label{eq:actual_best_response_notation}
\end{equation}
denote the best response selected by the algorithm. We use the same
deterministic tie-breaking rule throughout. Define the auxiliary response
\begin{equation}
    b_{\epsilon-\xi}^{\star}(\boldsymbol{\alpha},\theta)
    \in
    \argmin_{\lambda\in\Alt_{\epsilon-\xi}(\theta^\star)}
    \sum_{a\in\cA}
        \alpha_a k_a(\theta,\lambda).
    \label{eq:auxiliary_best_response_notation}
\end{equation}
The associated best-response modulus is
\begin{equation}
\begin{aligned}
    \widetilde u_\epsilon(\xi)
    :=
    \sup\Bigl\{
        &\left|
            k_a\!\left(
                \vartheta,
                b_\epsilon(\boldsymbol{\alpha},\theta)
            \right)
            -
            k_a\!\left(
                \vartheta,
                b_{\epsilon-\xi}^{\star}
                (\boldsymbol{\alpha},\theta)
            \right)
        \right|:\\
        &a\in\cA,\ 
        \boldsymbol{\alpha}\in\Delta(\cA),\
        \mathsf d_\Theta(\theta,\theta^\star)\leq\xi,\
        \mathsf d_\Theta(\vartheta,\theta^\star)\leq\xi
    \Bigr\}.
    \label{eq:best_response_continuity_modulus}
\end{aligned}
\end{equation}
Since the parameter class and allocation simplex are compact, the Gaussian
KL divergence is uniformly continuous, and the same deterministic
tie-breaking rule is used for both best responses, we have
\begin{equation}
    u(\xi)\longrightarrow0,
    \qquad
    \widetilde u_\epsilon(\xi)\longrightarrow0,
    \qquad
    C_{\epsilon-\xi}(\theta^\star)
    \longrightarrow C_\epsilon(\theta^\star)
    \quad\text{as }\xi\downarrow0.
    \label{eq:continuity_terms_vanish}
\end{equation}

\subsection{Linear growth of the stopping statistic}
\label{app:linear_growth_stopping_statistic}

\begin{lemma}[Pathwise linear information growth]
\label{lem:pathwise_linear_growth}
Fix $0<\xi<\min\{\epsilon,1\}$ and let
$T_0:=\tau_{\mathrm{loc}}(\xi)$. For every $T>T_0$,
\begin{align}
    d_T
    \geq\;&
    (T-T_0)
    \left[
        C_{\epsilon-\xi}(\theta^\star)
        -
        3u(\xi)
        -
        2\widetilde u_\epsilon(\xi)
    \right]
    -
    \mathcal R_T^{\mathrm{AH}}
    -
    T_0B_{\mathrm{KL}}
    \notag\\
    &\quad
    -
    K(K-1)(\sqrt T+2)B_{\mathrm{KL}}.
    \label{eq:pathwise_linear_growth}
\end{align}
\end{lemma}

\begin{proof}
By the definition of $T_0$, for every $t\geq T_0$,
\begin{equation}
    \mathsf d_\Theta(\widehat\theta_t,\theta^\star)<\xi.
    \label{eq:localized_estimates}
\end{equation}
In particular, since $\xi<1$, all these estimates have the correct graph:
$\widehat D_t=D^\star$.

For every $t>T_0$, write
\[
    \lambda_t
    :=
    b_\epsilon(\boldsymbol{\alpha}_t,\widehat\theta_t)
\]
for the actual alternative selected by the algorithm, and define the
auxiliary response
\[
    \lambda_t^\circ
    :=
    b_{\epsilon-\xi}^{\star}
    (\boldsymbol{\alpha}_t,\widehat\theta_t).
\]
By construction,
$\lambda_t^\circ\in\Alt_{\epsilon-\xi}(\theta^\star)$.

We begin with the stopping statistic
\[
    d_T
    =
    \inf_{\lambda\in\Alt_\epsilon(\widehat\theta_T)}
    \sum_{a\in\cA}
        N_T(a)k_a(\widehat\theta_T,\lambda).
\]
Using the lower tracking bound in
Eq.~(\ref{eq:cumulative_tracking_bound}), nonnegativity of KL divergence,
and the uniform upper bound $B_{\mathrm{KL}}$, we obtain
\begin{align}
    d_T
    \geq\;&
    \inf_{\lambda\in\Alt_\epsilon(\widehat\theta_T)}
    \sum_{t=T_0+1}^T
    \sum_{a\in\cA}
        \alpha_{t,a}
        k_a(\widehat\theta_T,\lambda)
    \notag\\
    &\quad
    -
    K(K-1)(\sqrt T+2)B_{\mathrm{KL}}.
    \label{eq:tracking_applied_to_dt}
\end{align}
For nonnegative functions $q_t(\lambda)$,
$\inf_\lambda\sum_tq_t(\lambda)\geq\sum_t\inf_\lambda q_t(\lambda)$.
Therefore, the first term in
Eq.~(\ref{eq:tracking_applied_to_dt}) is at least
\[
    \sum_{t=T_0+1}^T
    \inf_{\lambda\in\Alt_\epsilon(\widehat\theta_T)}
    \sum_{a\in\cA}
        \alpha_{t,a}k_a(\widehat\theta_T,\lambda).
\]

Because
$\mathsf d_\Theta(\widehat\theta_T,\theta^\star)<\xi$, the triangle
inequality gives
\begin{equation}
    \Alt_\epsilon(\widehat\theta_T)
    \subseteq
    \Alt_{\epsilon-\xi}(\theta^\star).
    \label{eq:alternative_set_inclusion}
\end{equation}
Indeed, a graph alternative to $\widehat\theta_T$ is also a graph alternative
to $\theta^\star$, while for a same-graph alternative,
\[
    \|\bA^\lambda-\bA^\star\|_{\max}
    \geq
    \|\bA^\lambda-\widehat{\bA}_T\|_{\max}
    -
    \|\widehat{\bA}_T-\bA^\star\|_{\max}
    \geq
    \epsilon-\xi.
\]

Both $\widehat\theta_T$ and $\widehat\theta_t$ lie in the
$\xi$-neighborhood of $\theta^\star$. Applying
Eq.~(\ref{eq:first_kl_continuity_modulus}) twice gives
\[
    \left|
        k_a(\widehat\theta_T,\lambda)
        -
        k_a(\widehat\theta_t,\lambda)
    \right|
    \leq 2u(\xi).
\]
Combining this inequality with
Eq.~(\ref{eq:alternative_set_inclusion}), for every $t>T_0$,
\begin{align}
    &\inf_{\lambda\in\Alt_\epsilon(\widehat\theta_T)}
    \sum_{a\in\cA}
        \alpha_{t,a}k_a(\widehat\theta_T,\lambda)
    \notag\\
    &\qquad\geq
    \sum_{a\in\cA}
        \alpha_{t,a}
        k_a(\widehat\theta_t,\lambda_t^\circ)
    -
    2u(\xi).
    \label{eq:replace_terminal_estimate}
\end{align}
By the definition of $\widetilde u_\epsilon(\xi)$,
\[
    k_a(\widehat\theta_t,\lambda_t^\circ)
    \geq
    k_a(\widehat\theta_t,\lambda_t)
    -
    \widetilde u_\epsilon(\xi).
\]
Substituting this into Eq.~(\ref{eq:replace_terminal_estimate}) and summing
over $t$ gives
\begin{align}
    d_T
    \geq\;&
    \sum_{t=T_0+1}^T
    \sum_{a\in\cA}
        \alpha_{t,a}
        k_a(\widehat\theta_t,\lambda_t)
    \notag\\
    &-
    (T-T_0)
    \bigl[2u(\xi)+\widetilde u_\epsilon(\xi)\bigr]
    -
    K(K-1)(\sqrt T+2)B_{\mathrm{KL}}.
    \label{eq:dt_before_adahedge}
\end{align}

Since
$r_{t,a}=k_a(\widehat\theta_t,\lambda_t)$, the truncated AdaHedge bound in
Eq.~(\ref{eq:truncated_adahedge_bound}) implies
\begin{align}
    d_T
    \geq\;&
    \max_{a\in\cA}
    \sum_{t=T_0+1}^T
        k_a(\widehat\theta_t,\lambda_t)
    -
    \mathcal R_T^{\mathrm{AH}}
    -
    T_0B_{\mathrm{KL}}
    \notag\\
    &-
    (T-T_0)
    \bigl[2u(\xi)+\widetilde u_\epsilon(\xi)\bigr]
    -
    K(K-1)(\sqrt T+2)B_{\mathrm{KL}}.
    \label{eq:dt_after_adahedge}
\end{align}

We next replace the first KL argument by $\theta^\star$. From the definition
of $u(\xi)$,
\[
    k_a(\widehat\theta_t,\lambda_t)
    \geq
    k_a(\theta^\star,\lambda_t)-u(\xi).
\]
The definition of $\widetilde u_\epsilon(\xi)$ also gives
\[
    k_a(\theta^\star,\lambda_t)
    \geq
    k_a(\theta^\star,\lambda_t^\circ)
    -
    \widetilde u_\epsilon(\xi).
\]
Consequently,
\begin{align}
    \max_{a\in\cA}
    \sum_{t=T_0+1}^T
        k_a(\widehat\theta_t,\lambda_t)
    \geq\;&
    \max_{a\in\cA}
    \sum_{t=T_0+1}^T
        k_a(\theta^\star,\lambda_t^\circ)
    \notag\\
    &-
    (T-T_0)
    \bigl[u(\xi)+\widetilde u_\epsilon(\xi)\bigr].
    \label{eq:replace_kl_by_true_parameter}
\end{align}

Let $\nu_T$ be the empirical probability measure that assigns mass
$1/(T-T_0)$ to each $\lambda_t^\circ$, $t=T_0+1,\ldots,T$. Since every
$\lambda_t^\circ$ belongs to
$\Alt_{\epsilon-\xi}(\theta^\star)$,
Lemma~\ref{lem:mixed_alternative_representation} yields
\begin{align}
    \max_{a\in\cA}
    \sum_{t=T_0+1}^T
        k_a(\theta^\star,\lambda_t^\circ)
    &=
    (T-T_0)
    \max_{a\in\cA}
    \int k_a(\theta^\star,\lambda)\,
        \nu_T(\mathrm d\lambda) \notag\\
    &\geq
    (T-T_0)C_{\epsilon-\xi}(\theta^\star).
    \label{eq:minimax_applied_to_responses}
\end{align}
Combining Eqs.~(\ref{eq:dt_after_adahedge}),
(\ref{eq:replace_kl_by_true_parameter}), and
(\ref{eq:minimax_applied_to_responses}) proves
Eq.~(\ref{eq:pathwise_linear_growth}).
\end{proof}

\subsection{Completion of the proof}
\label{app:final_upper_bound}

\begin{proof}[Proof of Theorem~\ref{thm:upper_bound}]
Fix $0<\xi<\min\{\epsilon,1\}$ and define
\begin{equation}
    c_\xi
    :=
    C_{\epsilon-\xi}(\theta^\star)
    -
    3u(\xi)
    -
    2\widetilde u_\epsilon(\xi).
    \label{eq:localized_information_rate}
\end{equation}
By Eq.~(\ref{eq:continuity_terms_vanish}), $\xi$ can be chosen sufficiently
small so that $c_\xi>0$.

Let $T_0:=\tau_{\mathrm{loc}}(\xi)$. The pathwise information-growth bound in
Eq.~(\ref{eq:pathwise_linear_growth}) can be written as
\begin{equation}
    d_T
    \geq
    c_\xi T
    -
    (c_\xi+B_{\mathrm{KL}})T_0
    -
    q_T,
    \label{eq:simplified_linear_growth}
\end{equation}
where
\[
    q_T
    :=
    \mathcal R_T^{\mathrm{AH}}
    +
    K(K-1)(\sqrt T+2)B_{\mathrm{KL}}.
\]
Lemma~\ref{lem:adahedge_regret} implies that $q_T=o(T)$ deterministically.

Fix any $\gamma\in(0,c_\xi)$. There exists a deterministic constant
$T_{\xi,\gamma}^{(0)}<\infty$ such that
$q_T\leq\gamma T/2$ for every
$T\geq T_{\xi,\gamma}^{(0)}$. Therefore, whenever
\[
    T
    \geq
    \max\left\{
        T_{\xi,\gamma}^{(0)},
        \frac{2(c_\xi+B_{\mathrm{KL}})}{\gamma}
        \tau_{\mathrm{loc}}(\xi)
    \right\},
\]
Eq.~(\ref{eq:simplified_linear_growth}) yields
\begin{equation}
    d_T\geq(c_\xi-\gamma)T.
    \label{eq:eventual_linear_lower_bound}
\end{equation}

From the explicit form of $f_T$ in
Eq.~(\ref{eq:concentration_function}), there exist deterministic constants
$C_f,m_f<\infty$, depending only on the known parameter bounds, such that
\[
    f_T(x)
    \leq
    C_f(T+x+2)^{m_f}e^{-\kappa_Bx}.
\]
Since the right-hand side is decreasing in $x$ for all sufficiently large
$T+x$, Eq.~(\ref{eq:eventual_linear_lower_bound}) gives
\begin{equation}
    \frac{\pi^2T^2}{6}f_T(d_T)
    \leq
    C_{\xi,\gamma}
    T^{m_f+2}
    e^{-\kappa_B(c_\xi-\gamma)T}
    \label{eq:stopping_envelope}
\end{equation}
for a deterministic constant $C_{\xi,\gamma}<\infty$ and all sufficiently
large $T$.

Let $\overline T_\delta(\xi,\gamma)$ be a deterministic time after which the
right-hand side of Eq.~(\ref{eq:stopping_envelope}) is smaller than
$\delta$. Taking logarithms shows that
\[
    \kappa_B(c_\xi-\gamma)\overline T_\delta(\xi,\gamma)
    =
    \log(1/\delta)
    +
    O\!\left(\log\log(1/\delta)\right),
\]
and hence
\begin{equation}
    \limsup_{\delta\downarrow0}
    \frac{\overline T_\delta(\xi,\gamma)}
         {\log(1/\delta)}
    \leq
    \frac{1}{\kappa_B(c_\xi-\gamma)}.
    \label{eq:deterministic_crossing_time}
\end{equation}

Consequently, for some deterministic constant
$C_{\xi,\gamma}^{(0)}<\infty$,
\begin{equation}
    \tau_\delta
    \leq
    \overline T_\delta(\xi,\gamma)
    +
    \frac{2(c_\xi+B_{\mathrm{KL}})}{\gamma}
    \tau_{\mathrm{loc}}(\xi)
    +
    C_{\xi,\gamma}^{(0)}.
    \label{eq:pathwise_stopping_time_upper_bound}
\end{equation}
Since
$\bbE_{\theta^\star}[\tau_{\mathrm{loc}}(\xi)]<\infty$ by
Lemma~\ref{lem:finite_expected_localization}, this bound first implies
$\Pp_{\theta^\star}(\tau_\delta<\infty)=1$. Taking expectations in
Eq.~(\ref{eq:pathwise_stopping_time_upper_bound}), dividing by
$\log(1/\delta)$, and using
Eq.~(\ref{eq:deterministic_crossing_time}) gives
\[
    \limsup_{\delta\downarrow0}
    \frac{\bbE_{\theta^\star}[\tau_\delta]}
         {\log(1/\delta)}
    \leq
    \frac{1}{\kappa_B(c_\xi-\gamma)}.
\]

Finally, letting $\gamma\downarrow0$ and then $\xi\downarrow0$, together with
Eq.~(\ref{eq:continuity_terms_vanish}), yields
\[
    \limsup_{\delta\downarrow0}
    \frac{\bbE_{\theta^\star}[\tau_\delta]}
         {\log(1/\delta)}
    \leq
    \frac{1}{\kappa_BC_\epsilon(\theta^\star)}.
\]
This proves the theorem.
\end{proof}

\section{Finite Implementations of the Optimization Oracles}
\label{app:oracles}

The theoretical description of \textsc{FOCUS} uses a joint-MLE oracle and an
alternative-model best-response oracle. This section provides finite
implementations of both procedures. The resulting algorithms require only
finite enumeration and convex quadratic programming. They are exponential in
$p$, as expected for unrestricted DAG search, but remove the need for abstract
optimization oracles.

For an action $a\in\cA$, let $\iota(a)$ denote its intervention target, with
$\iota(0)=\varnothing$. Because $\bX_s$ in the main text contains only the
nondegenerate random coordinates, let
$\widetilde{\bX}_s\in\mathbb R^p$ denote the corresponding full vector obtained
by reinserting the known intervention value when $a_s\neq0$. Thus,
$\widetilde X_{s,j}$ equals the assigned intervention value when
$j=\iota(a_s)$ and equals the observed random coordinate otherwise.

\subsection{Finite joint-MLE implementation}
\label{app:joint_mle_implementation}

\begin{algorithm}[htb]
\caption{Finite implementation of the joint-MLE oracle}
\label{alg:finite_joint_mle}
\begin{algorithmic}[1]
\Require History $\{(a_s,\bX_s)\}_{s=1}^t$ and parameter bounds
$\beta_{\min}$, $A_{\max}$, $\sigma_{\min}$, $\sigma_{\max}$
\Ensure A global joint MLE
$\widehat\theta_t=(\widehat D_t,\widehat{\bA}_t,
\widehat{\bSigma}_t)$

\State Reinsert each known intervention value to construct
$\{\widetilde{\bX}_s\}_{s=1}^t$
\State $\mathcal L_{\mathrm{best}}\gets+\infty$
\State $\widehat\theta_t\gets\text{null}$

\ForAll{permutations $\pi=(\pi_1,\ldots,\pi_p)$ of $[p]$}
    \State Initialize an empty candidate graph $D_\pi$
    \State $\mathcal L_\pi\gets0$

    \For{$m=1,\ldots,p$}
        \State $j\gets\pi_m$
        \State $\operatorname{Pred}_\pi(j)
        \gets\{\pi_1,\ldots,\pi_{m-1}\}$
        \State $\mathcal L_j^{\mathrm{best}}\gets+\infty$

        \ForAll{$S\subseteq\operatorname{Pred}_\pi(j)$}
            \If{$S=\varnothing$}
                \State Set $\bb$ to the empty vector and compute
                $Q_{j,t}(\bb;\varnothing)$
                \State Compute $v_j$ and $\mathcal L_{j,t}$ from
                Eqs.~(\ref{eq:local_mle_variance})--(\ref{eq:local_mle_score})
                \State Store $S=\varnothing$, $\bb$, $v_j$, and the local
                score as the best local solution if
                $\mathcal L_{j,t}<\mathcal L_j^{\mathrm{best}}$
            \Else
                \ForAll{sign vectors
                $\boldsymbol{\nu}\in\{-1,+1\}^{|S|}$}
                    \State Solve the convex quadratic program
                    \[
                        \bb^\star
                        \in
                        \argmin_{\bb\in\mathbb R^{|S|}}
                        Q_{j,t}(\bb;S)
                        \quad
                        \text{s.t.}\quad
                        \beta_{\min}
                        \leq \nu_k b_k
                        \leq A_{\max},
                        \ \ k\in S
                    \]
                    \If{the quadratic program is feasible}
                        \State Compute
                        $v_j^\star=v_{j,t}^{\star}(\bb^\star;S)$
                        and
                        $\mathcal L_{j,t}(\bb^\star;S)$
                        \State Store $S$, $\bb^\star$, $v_j^\star$, and the
                        local score as the best local solution if its score is
                        smaller than $\mathcal L_j^{\mathrm{best}}$
                    \EndIf
                \EndFor
            \EndIf
        \EndFor

        \State Set $\pa_{D_\pi}(j)$, $\bA_{j,\pa_{D_\pi}(j)}$, and
        $\sigma_j^2$ to the best local solution
        \State $\mathcal L_\pi\gets
        \mathcal L_\pi+\mathcal L_j^{\mathrm{best}}$
    \EndFor

    \If{$\mathcal L_\pi<\mathcal L_{\mathrm{best}}$}
        \State Store
        $\widehat\theta_t\gets
        (D_\pi,\bA_\pi,\bSigma_\pi)$
        \State $\mathcal L_{\mathrm{best}}\gets\mathcal L_\pi$
    \EndIf
\EndFor

\State \Return $\widehat\theta_t$
\end{algorithmic}
\end{algorithm}

Under a candidate DAG $D$, the density of a sample collected under action
$a_s$ factorizes over the nodes that were not intervened upon. Ignoring
constants independent of $\theta$, its negative log-likelihood is
\begin{equation}
    -\ell_t(\theta)
    =
    \frac{1}{2}
    \sum_{j=1}^p
    \sum_{\substack{s\leq t:\\\iota(a_s)\neq j}}
    \left[
        \log\sigma_j^2
        +
        \frac{
        \left(
            \widetilde X_{s,j}
            -
            \bA_{j,\pa_D(j)}
            \widetilde{\bX}_{s,\pa_D(j)}
        \right)^2
        }{\sigma_j^2}
    \right].
    \label{eq:intervention_masked_likelihood}
\end{equation}
The structural equation of an intervened node is excluded because it is
replaced by the intervention assignment. Nevertheless, that fixed coordinate
remains available as a regressor in the structural equations of the other
nodes.

For a node $j$, define
\[
    \mathcal I_{j,t}
    :=
    \{s\in[t]:\iota(a_s)\neq j\},
    \qquad
    n_{j,t}:=|\mathcal I_{j,t}|.
\]
Given a proposed parent set $S\subseteq[p]\setminus\{j\}$ and coefficient
vector $\bb\in\mathbb R^{|S|}$, let
\begin{equation}
    Q_{j,t}(\bb;S)
    :=
    \sum_{s\in\mathcal I_{j,t}}
    \left(
        \widetilde X_{s,j}
        -
        \bb^{\mathsf T}\widetilde{\bX}_{s,S}
    \right)^2.
    \label{eq:local_residual_sum_squares}
\end{equation}
For fixed $\bb$ and $S$, the optimal noise variance is
\begin{equation}
    v_{j,t}^{\star}(\bb;S)
    :=
    \Pi_{[\sigma_{\min}^2,\sigma_{\max}^2]}
    \left(
        \frac{Q_{j,t}(\bb;S)}{n_{j,t}}
    \right),
    \label{eq:local_mle_variance}
\end{equation}
where $\Pi_{[l,u]}(x):=\min\{u,\max\{l,x\}\}$. If $n_{j,t}=0$, the local
likelihood is constant and we set
$v_{j,t}^{\star}(\bb;S)=\sigma_{\min}^2$ by convention.

Substituting Eq.~(\ref{eq:local_mle_variance}) into the likelihood gives the
local score
\begin{equation}
    \mathcal L_{j,t}(\bb;S)
    :=
    \frac{1}{2}
    \left[
        n_{j,t}\log v_{j,t}^{\star}(\bb;S)
        +
        \frac{
            Q_{j,t}(\bb;S)
        }{
            v_{j,t}^{\star}(\bb;S)
        }
    \right].
    \label{eq:local_mle_score}
\end{equation}
The optimized score is nondecreasing in
$Q_{j,t}(\bb;S)$. Therefore, for a fixed support and coefficient-sign pattern,
it suffices to minimize the convex quadratic function in
Eq.~(\ref{eq:local_residual_sum_squares}).

\paragraph{Why Algorithm~\ref{alg:finite_joint_mle} is exact.}
For a fixed topological order, every parent of node $j$ must occur among its
predecessors. Enumerating all predecessor subsets therefore enumerates every
DAG compatible with that order. Enumerating the coefficient signs converts
the nonconvex constraint
$\beta_{\min}\leq|A_{jk}|\leq A_{\max}$ into finitely many linear box
constraints. Conditional on the order, parent set, and signs, the remaining
coefficient problem is convex.

Moreover, Eq.~(\ref{eq:intervention_masked_likelihood}) decomposes over
nodes. Hence the best parent set and parameters for each node can be selected
independently once the order is fixed. Finally, every DAG admits at least one
topological order, so enumeration over all permutations includes every
parameter in $\Theta$. Algorithm~\ref{alg:finite_joint_mle} consequently
returns a global maximizer of Eq.~(\ref{eq:joint_mle}). Duplicate appearances
of the same DAG under different topological orders do not affect correctness.

\subsection{Local representation of the alternative objective}
\label{app:alternative_local_scores}

The second finite algorithm computes both the allocation best response in
Eq.~(\ref{eq:alternative_best_response}) and the stopping statistic in
Eq.~(\ref{eq:stopping_statistic}). Its generic inputs are:
\[
    \overline\theta
    =
    (\overline D,\overline{\bA},\overline{\bSigma})
    \in\Theta,
    \qquad
    \bw=(w_a)_{a\in\cA}\in\mathbb R_+^K,
    \qquad
    r>0.
\]
It solves
\begin{equation}
    \inf_{\lambda\in\Alt_r(\overline\theta)}
    \sum_{a\in\cA}
        w_a\KL(P_a^{\overline\theta}\|P_a^\lambda).
    \label{eq:generic_alternative_problem}
\end{equation}
Taking $\bw=\boldsymbol{\alpha}_t$ and $r=\epsilon$ gives the best response
$\lambda_t$. Taking $w_a=N_t(a)$ and $r=\epsilon$ gives $d_t$ and its
minimizing alternative.

For each action $a$, compute analytically under the reference parameter
$\overline\theta$ the full-vector moments
\begin{equation}
    \overline{\boldsymbol{\mu}}^{(a)}
    :=
    \bbE_{\overline\theta,a}[\bX],
    \qquad
    \overline{\bH}^{(a)}
    :=
    \bbE_{\overline\theta,a}[\bX\bX^{\mathsf T}]
    =
    \overline{\bV}^{(a)}
    +
    \overline{\boldsymbol{\mu}}^{(a)}
    \overline{\boldsymbol{\mu}}^{(a)\mathsf T}.
    \label{eq:reference_action_moments}
\end{equation}
For an intervention, the fixed coordinate is reinserted into the mean and
second-moment matrix and has zero conditional variance. These quantities are
available explicitly from Eq.~(\ref{eq:interventional_distribution}).

For a proposed parent set $S$ of node $j$, define
\begin{align}
    W_j
    &:=
    \sum_{\substack{a\in\cA:\\\iota(a)\neq j}}w_a,
    \label{eq:alternative_local_weight}\\
    \bG_j(S)
    &:=
    \sum_{\substack{a\in\cA:\\\iota(a)\neq j}}
        w_a\overline{\bH}^{(a)}_{S,S},
    \label{eq:alternative_local_gram}\\
    \bh_j(S)
    &:=
    \sum_{\substack{a\in\cA:\\\iota(a)\neq j}}
        w_a\overline{\bH}^{(a)}_{S,j},
    \label{eq:alternative_local_cross_moment}\\
    g_j
    &:=
    \sum_{\substack{a\in\cA:\\\iota(a)\neq j}}
        w_a\overline H^{(a)}_{j,j}.
    \label{eq:alternative_local_second_moment}
\end{align}
The weighted expected squared residual associated with coefficient vector
$\bb\in\mathbb R^{|S|}$ is
\begin{equation}
    R_j(\bb;S)
    :=
    g_j
    -
    2\bb^{\mathsf T}\bh_j(S)
    +
    \bb^{\mathsf T}\bG_j(S)\bb.
    \label{eq:alternative_expected_residual}
\end{equation}
Since $\bG_j(S)$ is positive semidefinite, this is a convex quadratic
function of $\bb$.

For a candidate conditional variance $v$, the local contribution to the
objective in Eq.~(\ref{eq:generic_alternative_problem}) is
\begin{equation}
    \Phi_j(\bb,v;S)
    :=
    \frac{W_j}{2}
    \log\frac{v}{\overline\sigma_j^2}
    +
    \frac{R_j(\bb;S)}{2v}
    -
    \frac{W_j}{2}.
    \label{eq:alternative_local_kl_score}
\end{equation}
For a candidate parameter
$\lambda=(D^\lambda,\bA^\lambda,\bSigma^\lambda)$, let
$S_j=\pa_{D^\lambda}(j)$,
$\bb_j=\bA^\lambda_{j,S_j}$, and
$v_j=(\sigma_j^\lambda)^2$. The weighted KL objective decomposes as
\begin{equation}
    \sum_{a\in\cA}
        w_a\KL(P_a^{\overline\theta}\|P_a^\lambda)
    =
    \sum_{j=1}^p
        \Phi_j(\bb_j,v_j;S_j).
    \label{eq:alternative_objective_decomposition}
\end{equation}
When $W_j>0$, its minimizing variance is
\begin{equation}
    v_j^\star(\bb;S)
    =
    \Pi_{[\sigma_{\min}^2,\sigma_{\max}^2]}
    \left(
        \frac{R_j(\bb;S)}{W_j}
    \right).
    \label{eq:alternative_optimal_variance}
\end{equation}
The optimized value of Eq.~(\ref{eq:alternative_local_kl_score}) is
nondecreasing in $R_j(\bb;S)$. Thus, for each support and sign pattern, the
coefficient vector is found by minimizing the convex quadratic function
$R_j(\bb;S)$ and the variance is subsequently obtained from
Eq.~(\ref{eq:alternative_optimal_variance}).

If $W_j=0$, the available weighted actions contain no likelihood contribution
from the structural equation of node $j$. We then set its local score to zero
and select any feasible coefficients and variance using a fixed deterministic
tie-breaking rule.

For later use, define
\begin{equation}
    \operatorname{LocalFit}(j,S,\mathcal C)
\end{equation}
as the following finite subroutine: enumerate all sign vectors
$\boldsymbol{\nu}\in\{-1,+1\}^{|S|}$, minimize $R_j(\bb;S)$ subject to
\[
    \beta_{\min}\leq \nu_kb_k\leq A_{\max},
    \qquad k\in S,
\]
together with any additional linear constraints in $\mathcal C$, discard
infeasible sign patterns, compute the optimal variance from
Eq.~(\ref{eq:alternative_optimal_variance}), and return the feasible
parameters having the smallest value of
Eq.~(\ref{eq:alternative_local_kl_score}). For $S=\varnothing$, the
coefficient vector is empty and $R_j(\bb;S)=g_j$.

Whenever the local problem is feasible, we write
\[
    \bigl(
        \bb_j^\star(S,\mathcal C),
        v_j^\star(S,\mathcal C),
        \Phi_j^\star(S,\mathcal C)
    \bigr)
    :=
    \operatorname{LocalFit}(j,S,\mathcal C)
\]
for its returned coefficients, variance, and optimized local score. When
$\mathcal C=\varnothing$, we abbreviate
$\Phi_j^\star(S,\varnothing)$ as $\Phi_j^\star(S)$.

\begin{algorithm}[hbt]
\caption{Finite implementation of the alternative-model oracle}
\label{alg:finite_alternative_oracle}
\begin{algorithmic}[1]
\Require Reference parameter
$\overline\theta=(\overline D,\overline{\bA},\overline{\bSigma})$,
weights $\bw=(w_a)_{a\in\cA}$, and radius $r$
\Ensure Global optimizer $(\lambda^\star,\mathcal V^\star)$ of
Eq.~(\ref{eq:generic_alternative_problem})

\State Compute
$\{\overline{\boldsymbol{\mu}}^{(a)},
\overline{\bH}^{(a)}\}_{a\in\cA}$
using Eq.~(\ref{eq:reference_action_moments})
\State Compute the local quantities in
Eqs.~(\ref{eq:alternative_local_weight})--%
(\ref{eq:alternative_expected_residual})
\State $(\lambda_A,\mathcal V_A)
\gets\Call{SameGraphAlternative}{\overline\theta,\bw,r}$
(Algorithm~\ref{alg:same_graph_alternative})
\State $(\lambda_B,\mathcal V_B)
\gets\Call{DifferentGraphAlternative}{\overline\theta,\bw}$
(Algorithm~\ref{alg:different_graph_alternative})

\If{$\mathcal V_A\leq\mathcal V_B$}
    \State $(\lambda^\star,\mathcal V^\star)
    \gets(\lambda_A,\mathcal V_A)$
\Else
    \State $(\lambda^\star,\mathcal V^\star)
    \gets(\lambda_B,\mathcal V_B)$
\EndIf

\State \Return $(\lambda^\star,\mathcal V^\star)$
\end{algorithmic}
\end{algorithm}

\begin{algorithm}[htb]
\caption{\textsc{SameGraphAlternative}: same-graph branch}
\label{alg:same_graph_alternative}
\begin{algorithmic}[1]
\Require Reference parameter $\overline\theta$, weights $\bw$, and radius $r$
\Ensure Best same-graph alternative $(\lambda_A,\mathcal V_A)$

\State $\mathcal V_A\gets+\infty$ and $\lambda_A\gets\text{null}$

\For{$j=1,\ldots,p$}
    \State $\overline S_j\gets\pa_{\overline D}(j)$
    \ForAll{$k\in\overline S_j$ and $d\in\{-1,+1\}$}
        \State Define
        $\mathcal C_{j,k,d}
        :=\{d(b_k-\overline A_{jk})\geq r\}$
        \State Call $\operatorname{LocalFit}
        (j,\overline S_j,\mathcal C_{j,k,d})$
        \If{the local problem is feasible}
            \State Let $(\bb^\star,v^\star,
            \Phi_j^\star(\overline S_j,\mathcal C_{j,k,d}))$
            be the returned local fit
            \State Construct $\lambda$ using the returned row and variance
            for node $j$
            \State Keep all other coefficient rows and noise variances equal
            to their reference values
            \State Let $\mathcal V$ be the returned local score
            \If{$\mathcal V<\mathcal V_A$}
                \State $(\lambda_A,\mathcal V_A)
                \gets(\lambda,\mathcal V)$
            \EndIf
        \EndIf
    \EndFor
\EndFor

\State \Return $(\lambda_A,\mathcal V_A)$
\end{algorithmic}
\end{algorithm}

\subsection{Finite alternative-model implementation}
\label{app:alternative_oracle_implementation}

The alternative set in Eq.~(\ref{eq:alternative_set}) is the union of two
branches:
\[
    D^\lambda=\overline D,\quad
    \|\bA^\lambda-\overline{\bA}\|_{\max}\geq r,
\]
and
\[
    D^\lambda\neq\overline D.
\]
Algorithms~\ref{alg:same_graph_alternative} and
\ref{alg:different_graph_alternative} solve the two branches separately,
and Algorithm~\ref{alg:finite_alternative_oracle} returns the better solution.

\paragraph{Explanation of Algorithm~\ref{alg:same_graph_alternative}.}
When $D^\lambda=\overline D$, membership in
$\Alt_r(\overline\theta)$ requires at least one existing coefficient to
differ from its reference value by at least $r$. This condition can be
written as the finite union
\[
    b_k-\overline A_{jk}\geq r
    \qquad\text{or}\qquad
    \overline A_{jk}-b_k\geq r
\]
over all existing edges $k\to j$. Each alternative is a linear constraint.

Only one structural row needs to be modified. Indeed, once node $j$ and edge
$k\to j$ satisfy the separation constraint, changing any other row is not
needed for feasibility. Keeping those rows equal to the reference parameter
gives zero local KL contribution, whereas optimizing them cannot reduce the
total KL below their reference value. The constrained problem for the
selected row is therefore sufficient to obtain the global optimum of the
same-graph branch.

The coefficient optimizer may change several entries in row $j$, even though
only one entry is required to deviate by at least $r$. This is intentional:
allowing all coefficients in that row to adjust can reduce the KL objective
while preserving the required separation.

\begin{algorithm}[htb]
\caption{\textsc{DifferentGraphAlternative}: different-graph branch}
\label{alg:different_graph_alternative}
\begin{algorithmic}[1]
\Require Reference parameter $\overline\theta$, weights $\bw$
\Ensure Best different-graph alternative $(\lambda_B,\mathcal V_B)$

\State $\mathcal V_B\gets+\infty$ and $\lambda_B\gets\text{null}$

\ForAll{permutations $\pi=(\pi_1,\ldots,\pi_p)$ of $[p]$}
    \State $F_0(0)\gets0$ and $F_0(1)\gets+\infty$
    \State Initialize backtracking records

    \For{$m=1,\ldots,p$}
        \State $j\gets\pi_m$ and
        $\operatorname{Pred}_\pi(j)
        \gets\{\pi_1,\ldots,\pi_{m-1}\}$
        \State $F_m(0),F_m(1)\gets+\infty$

        \ForAll{$q\in\{0,1\}$ such that $F_{m-1}(q)<+\infty$}
            \ForAll{$S\subseteq\operatorname{Pred}_\pi(j)$}
                \State Call $\operatorname{LocalFit}(j,S,\varnothing)$
                \If{the local problem is feasible}
                    \State Let
                    $(\bb^\star,v^\star,\Phi_j^\star(S))$
                    be the returned local fit
                    \State $q'\gets
                    \max\{q,\mathbf 1\{S\neq\pa_{\overline D}(j)\}\}$
                    \State $\mathcal V\gets
                    F_{m-1}(q)+\Phi_j^\star(S)$
                    \If{$\mathcal V<F_m(q')$}
                        \State $F_m(q')\gets\mathcal V$
                        \State Store $(S,\bb^\star,v^\star)$ and the predecessor state for $(m,q')$
                    \EndIf
                \EndIf
            \EndFor
        \EndFor
    \EndFor

    \If{$F_p(1)<\mathcal V_B$}
        \State Backtrack state $1$ to reconstruct $\lambda$
        \State $(\lambda_B,\mathcal V_B)
        \gets(\lambda,F_p(1))$
    \EndIf
\EndFor

\State \Return $(\lambda_B,\mathcal V_B)$
\end{algorithmic}
\end{algorithm}

\paragraph{Explanation of Algorithm~\ref{alg:different_graph_alternative}.}
For a fixed topological order, every candidate parent set must be a subset of
the node's predecessors. The local KL objective then decomposes across nodes,
but the global constraint $D^\lambda\neq\overline D$ couples the local
choices. Algorithm~\ref{alg:different_graph_alternative} handles this
constraint using a two-state dynamic program:
\[
    q=0:
    \text{all parent sets chosen so far agree with }\overline D,
\]
and
\[
    q=1:
    \text{at least one chosen parent set differs from }\overline D.
\]
Only state $q=1$ is accepted at the end. This prevents the reference graph
from being returned by the different-graph branch. Enumerating all
topological orders ensures that every candidate DAG is considered.

\paragraph{Using Algorithm~\ref{alg:finite_alternative_oracle} in
\textsc{FOCUS}.}
At round $t$, the allocation best response is obtained by calling
Algorithm~\ref{alg:finite_alternative_oracle} with
\[
    \overline\theta=\widehat\theta_t,
    \qquad
    w_a=\alpha_{t,a},
    \qquad
    r=\epsilon.
\]
The returned parameter is $\lambda_t$, and the action-wise gains are then
computed as
\[
    r_{t,a}
    =
    \KL(P_a^{\widehat\theta_t}\|P_a^{\lambda_t}).
\]

The stopping statistic is obtained from a second call with
\[
    \overline\theta=\widehat\theta_t,
    \qquad
    w_a=N_t(a),
    \qquad
    r=\epsilon.
\]
In this case, the returned objective value is exactly
\[
    \mathcal V^\star
    =
    \inf_{\lambda\in\Alt_\epsilon(\widehat\theta_t)}
    \sum_{a\in\cA}
        N_t(a)
        \KL(P_a^{\widehat\theta_t}\|P_a^\lambda)
    =
    d_t.
\]
Thus, the same finite procedure implements both alternative-model
optimization calls required by Algorithm~\ref{alg:focus}.

\paragraph{Computational cost.}
Both algorithms are finite but exponential in the number of nodes. The
joint-MLE procedure enumerates topological orders, parent sets, and
coefficient-sign patterns. The alternative-model procedure additionally
maintains the two graph-difference states. Every continuous subproblem is a
convex quadratic program with linear constraints and can therefore be solved
to global optimality.

These algorithms are intended to demonstrate that the oracle formulation
does not hide an ill-defined or noncomputable operation. For the small graphs
considered in our experiments, the finite procedures can be implemented
directly. For substantially larger graphs, restricted parent-set sizes,
order-based dynamic programming, or approximate continuous structure search
may be used to reduce the computational cost, although such replacements no
longer automatically inherit the exact-oracle guarantees of
Theorems~\ref{thm:focus_correctness} and~\ref{thm:upper_bound}.

\section{Additional Experimental Results}
\label{app:experiments}

Following the experimental protocol described in
Section~\ref{sec:experiments}, we consider five settings:
$(p,\rho)\in\{(5,0.3),(6,0.3),(7,0.3),(6,0.4),(6,0.5)\}$. The representative
setting $(p,\rho)=(6,0.3)$ is reported in the main text, while the remaining
four settings are shown in
Figures~\ref{fig:additional_p5_rho03}--\ref{fig:additional_p6_rho05}.
The additional results are consistent with the main-text experiment and lead
to the same conclusions.

\begin{figure*}[h]
    \centering
    \begin{subfigure}[t]{0.32\textwidth}
        \centering
        \includegraphics[width=\linewidth]{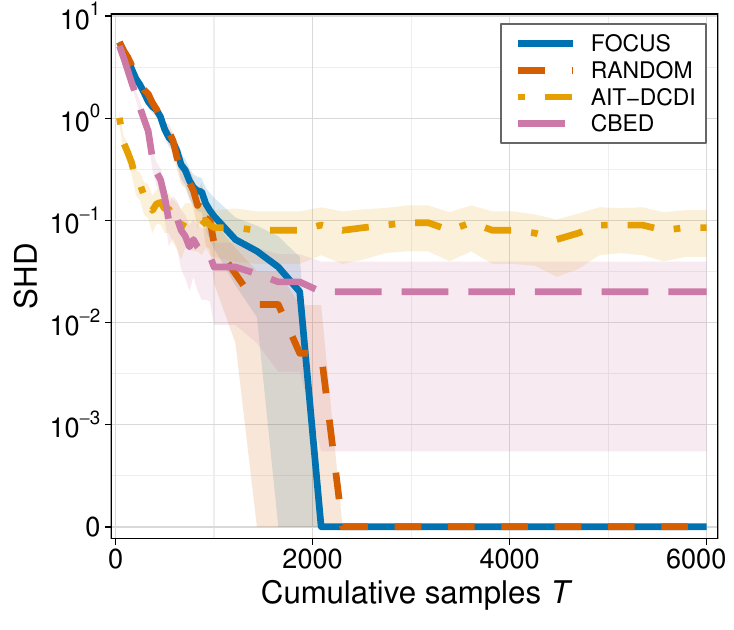}
        \caption{Structure error}
        
    \end{subfigure}
    \hfill
    \begin{subfigure}[t]{0.32\textwidth}
        \centering
        \includegraphics[width=\linewidth]{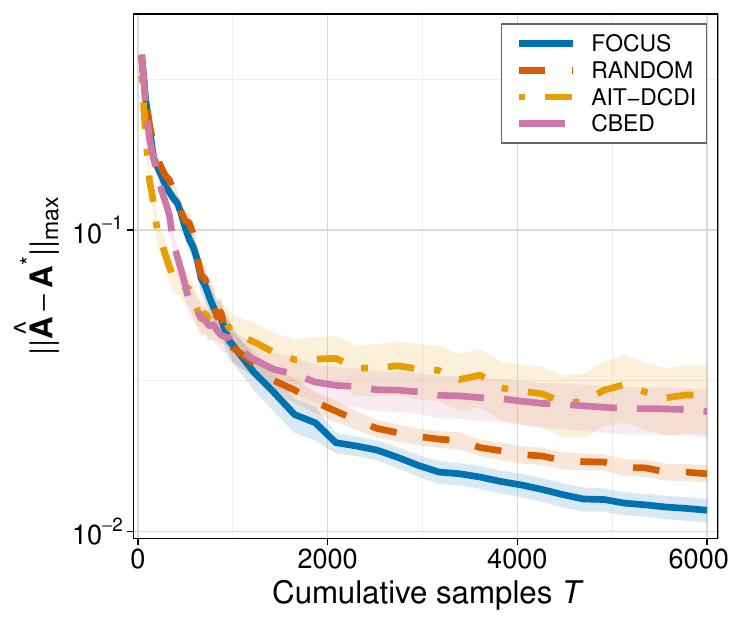}
        \caption{Edge-weight error}
        
    \end{subfigure}
    \hfill
    \begin{subfigure}[t]{0.32\textwidth}
        \centering
        \includegraphics[width=\linewidth]{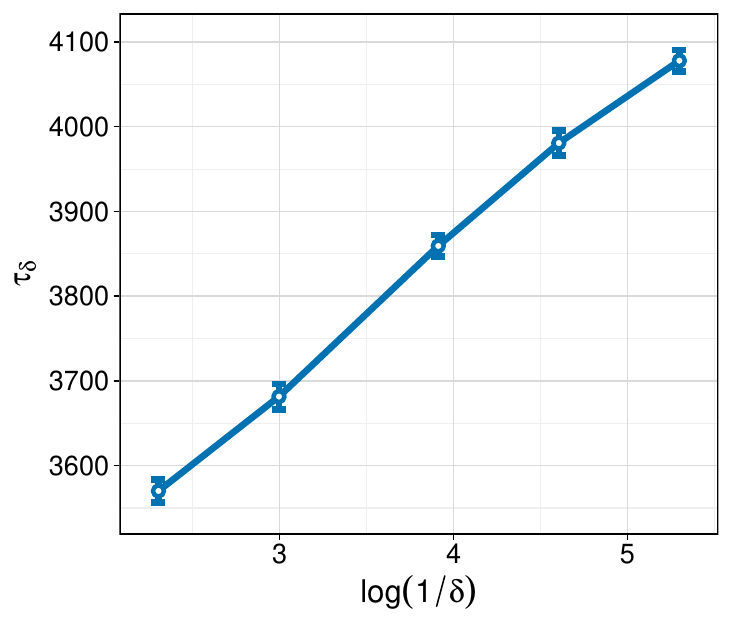}
        \caption{Stopping-time scaling}
        
    \end{subfigure}

    \caption{
        Performance for $p=5$ and $\rho=0.3$. Panels (a)--(c) show SHD, maximum
edge-weight estimation error, and the stopping time of \textsc{FOCUS} versus
$\log(1/\delta)$, respectively. Shaded bands and error bars denote 95\%
confidence intervals.
    }
    \label{fig:additional_p5_rho03}
\end{figure*}

\begin{figure*}[h]
    \centering
    \begin{subfigure}[t]{0.32\textwidth}
        \centering
        \includegraphics[width=\linewidth]{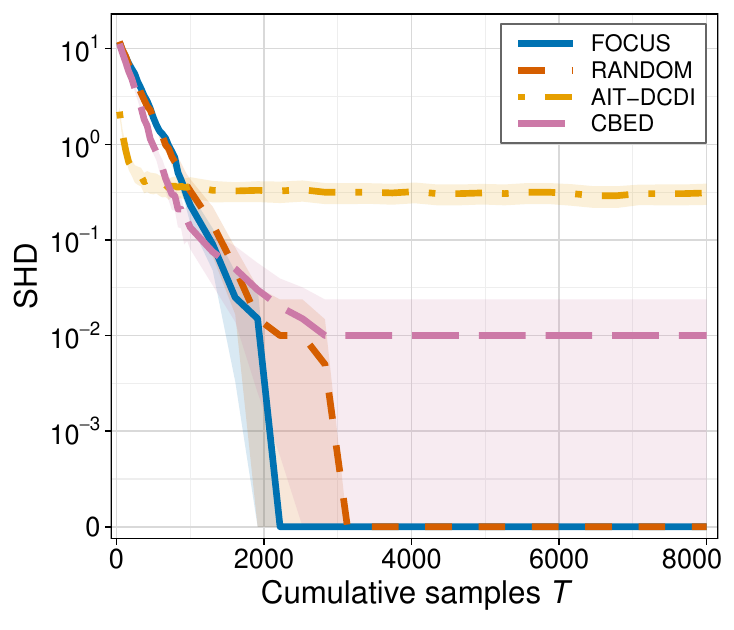}
        \caption{Structure error}
        
    \end{subfigure}
    \hfill
    \begin{subfigure}[t]{0.32\textwidth}
        \centering
        \includegraphics[width=\linewidth]{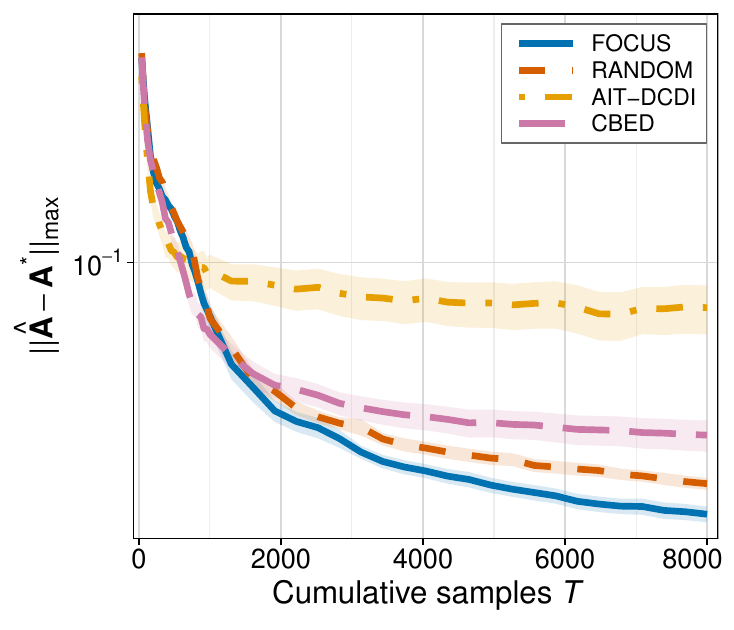}
        \caption{Edge-weight error}
        
    \end{subfigure}
    \hfill
    \begin{subfigure}[t]{0.32\textwidth}
        \centering
        \includegraphics[width=\linewidth]{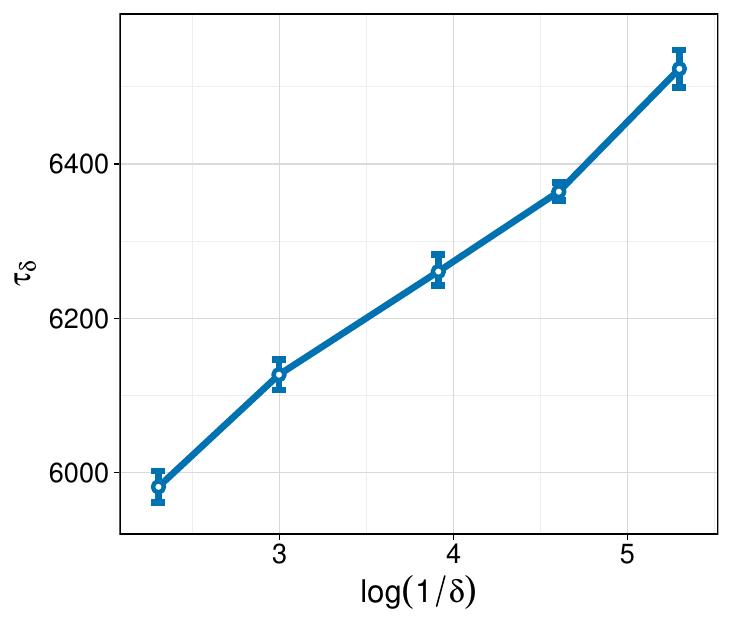}
        \caption{Stopping-time scaling}
        
    \end{subfigure}

    \caption{
        Performance for $p=7$ and $\rho=0.3$. Panels (a)--(c) show SHD, maximum
edge-weight estimation error, and the stopping time of \textsc{FOCUS} versus
$\log(1/\delta)$, respectively. Shaded bands and error bars denote 95\%
confidence intervals.
    }
    \label{fig:additional_p7_rho03}
\end{figure*}

\begin{figure*}[h]
    \centering
    \begin{subfigure}[t]{0.32\textwidth}
        \centering
        \includegraphics[width=\linewidth]{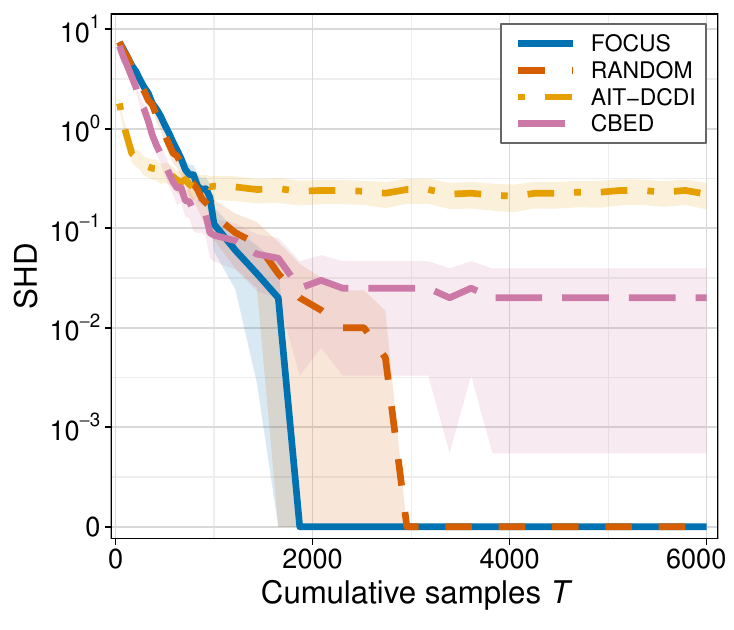}
        \caption{Structure error}
        
    \end{subfigure}
    \hfill
    \begin{subfigure}[t]{0.32\textwidth}
        \centering
        \includegraphics[width=\linewidth]{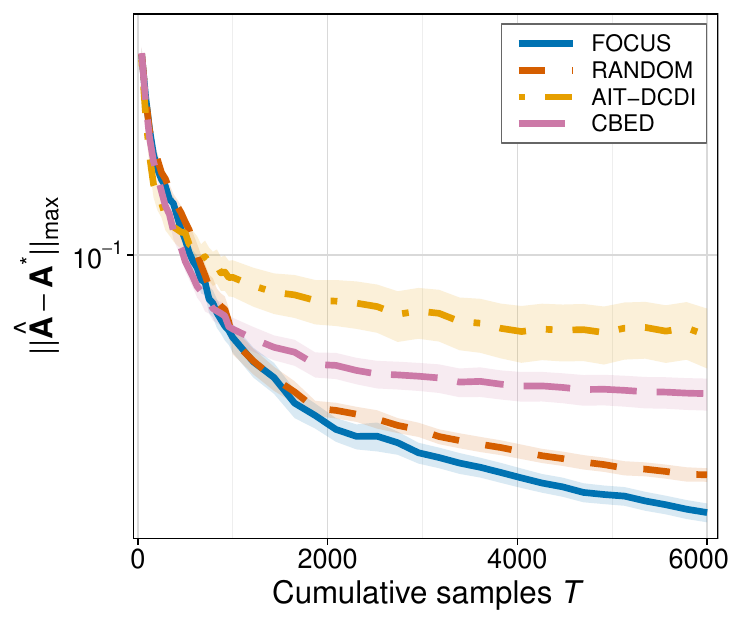}
        \caption{Edge-weight error}
        
    \end{subfigure}
    \hfill
    \begin{subfigure}[t]{0.32\textwidth}
        \centering
        \includegraphics[width=\linewidth]{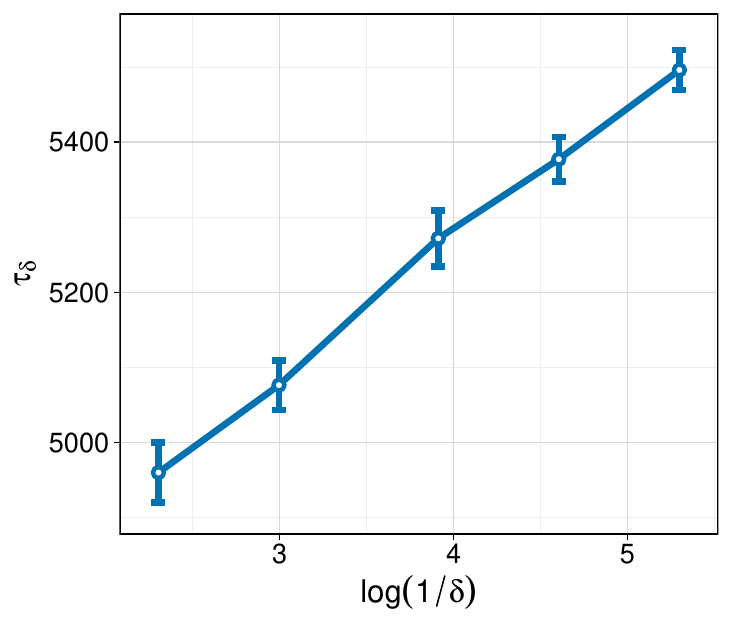}
        \caption{Stopping-time scaling}
        
    \end{subfigure}

    \caption{
        Performance for $p=6$ and $\rho=0.4$. Panels (a)--(c) show SHD, maximum
edge-weight estimation error, and the stopping time of \textsc{FOCUS} versus
$\log(1/\delta)$, respectively. Shaded bands and error bars denote 95\%
confidence intervals.
    }
    \label{fig:additional_p6_rho04}
\end{figure*}

\begin{figure*}[h]
    \centering
    \begin{subfigure}[t]{0.32\textwidth}
        \centering
        \includegraphics[width=\linewidth]{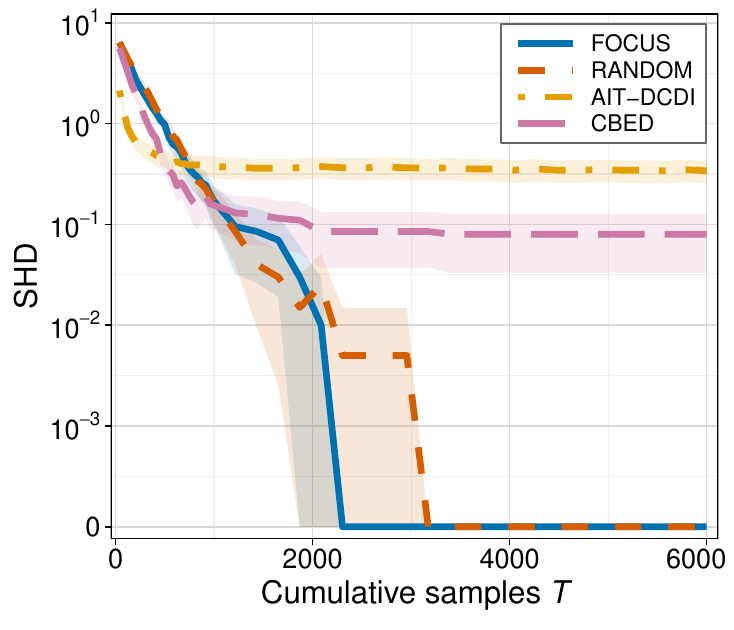}
        \caption{Structure error}
        
    \end{subfigure}
    \hfill
    \begin{subfigure}[t]{0.32\textwidth}
        \centering
        \includegraphics[width=\linewidth]{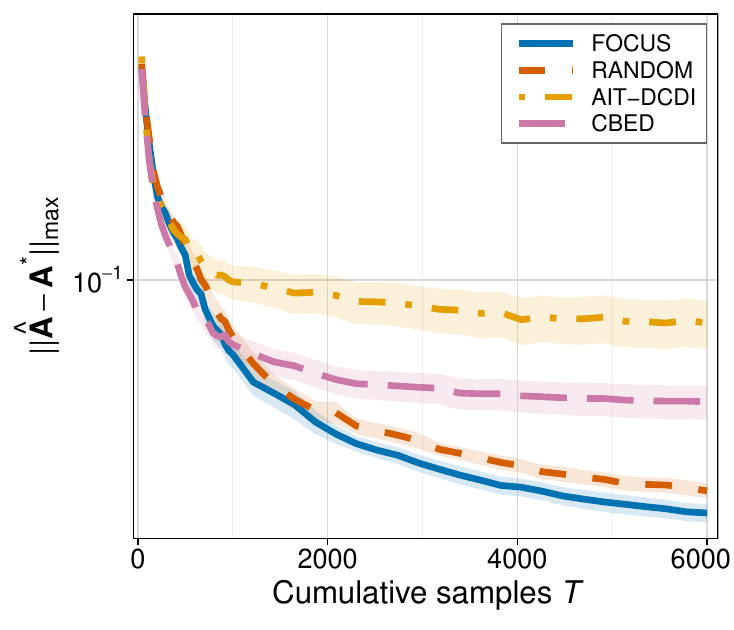}
        \caption{Edge-weight error}
        
    \end{subfigure}
    \hfill
    \begin{subfigure}[t]{0.32\textwidth}
        \centering
        \includegraphics[width=\linewidth]{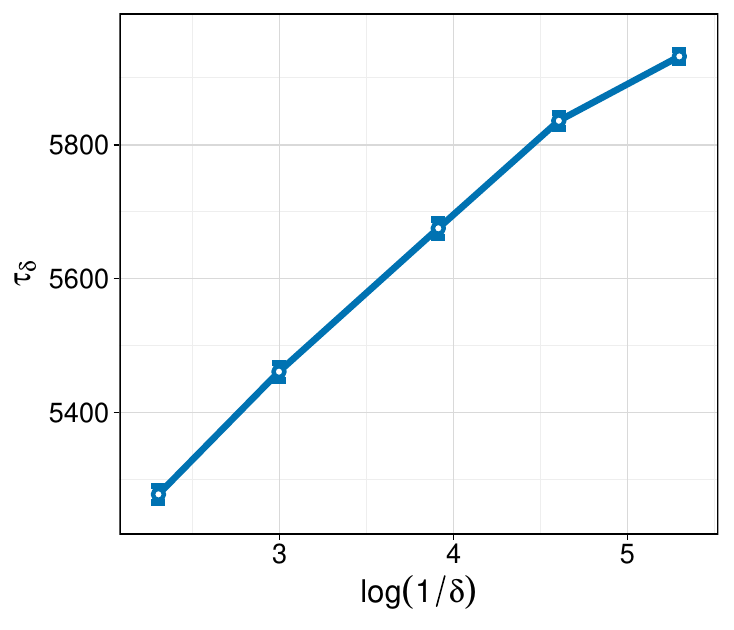}
        \caption{Stopping-time scaling}
        
    \end{subfigure}

    \caption{
        Performance for $p=6$ and $\rho=0.5$. Panels (a)--(c) show SHD, maximum
edge-weight estimation error, and the stopping time of \textsc{FOCUS} versus
$\log(1/\delta)$, respectively. Shaded bands and error bars denote 95\%
confidence intervals.
    }
    \label{fig:additional_p6_rho05}
\end{figure*}

\end{document}